\documentclass[10pt,journal,cspaper,compsoc]{IEEEtran}
\usepackage{graphicx}
\usepackage{amsmath,amssymb} 
\usepackage{lineno}
\usepackage{color}
\usepackage{epsfig}
\usepackage{url}
\usepackage[boxed,commentsnumbered,ruled,linesnumbered]{algorithm2e}
\usepackage{subfigure}
\usepackage{multirow}
\newtheorem{lemma}{Lemma}
\newtheorem{theorem}{Theorem}
\begin{document}
\title{A Fast Projected Fixed-Point Algorithm for Large Graph Matching}
\author{Yao Lu,
        Kaizhu Huang,
        and~Cheng-Lin Liu,~\IEEEmembership{Senior Member,~IEEE}
\IEEEcompsocitemizethanks{\IEEEcompsocthanksitem Yao Lu, Kaizhu Huang and Cheng-Lin Liu are with
National Laboratory of Pattern Recognition,
Institute of Automation, Chinese Academy of Sciences\protect\\
E-mail: yaolubrain@gmail.com, kzhuang@nlpr.ia.ac.cn, liucl@nlpr.ia.ac.cn
}%
\thanks{}
}

\markboth{}%
{}

\IEEEcompsoctitleabstractindextext{%
\begin{abstract}
We propose a fast approximate algorithm for large graph
matching. A new projected fixed-point
method is defined and a new doubly stochastic projection is adopted to derive the
algorithm. Previous graph matching algorithms suffer from high
computational complexity and therefore do not have good scalability
with respect to graph size. For matching two weighted graphs of $n$
nodes, our algorithm has time complexity only $O(n^3)$ per iteration and space
complexity $O(n^2)$. In addition to its scalability, our algorithm is
easy to implement, robust, and able to match undirected weighted attributed graphs of
different sizes. While the convergence rate of previous iterative graph matching algorithms is unknown, our algorithm is theoretically guaranteed to converge at a linear rate. Extensive experiments on large synthetic and real
graphs (more than 1,000 nodes) were conducted to evaluate the
performance of various algorithms. Results show that in most cases our proposed
algorithm achieves better performance than previous state-of-the-art algorithms in terms
of both speed and accuracy in large graph matching. In particular, with high
accuracy, our algorithm takes only a few seconds (in a PC) to match
two graphs of 1,000 nodes.
\end{abstract}
}

\maketitle
\IEEEdisplaynotcompsoctitleabstractindextext
\IEEEpeerreviewmaketitle

\section{Introduction}
\IEEEPARstart{G}raph matching, aiming to find the optimal correspondences between
the nodes of two graphs, is an important and active topic in
computer vision and pattern recognition~\cite{GM30,RecentAdvance}. It has been extensively
applied in various fields including optical character recognition~\cite{OCR1,OCR2}, object recognition~\cite{object,shape1}, shape matching~\cite{shape1,shape2,shape3}, face recognition~\cite{face}, feature correspondence~\cite{feature}, point matching~\cite{point}, image retrieval~\cite{retrieval}, and video indexing~\cite{video}.

Graph matching is in general an NP-hard discrete optimization
problem. Exact graph isomorphism algorithms include Ullman's method~\cite{Ullman}, Nauty~\cite{Nauty} and VF2~\cite{VF}, all of which have exponential time complexity in worst cases. To match
two graphs within a reasonable time, one has to look for approximate
solutions. Moreover, due to noise and variability in real world graphs, the usage of exact graph matching algorithms is very limited. The focus of this paper is the design of a approximate algorithm for efficiently matching general large graphs (e.g. graphs of 1,000 nodes) in computer vision and patter recognition.

One approach of approximate graph matching algorithms is based tree search~\cite{tree1,tree2}. Its basic idea is tree search with backtracking while using heurstics to prune unfrutful paths. Another approach of graph matching algorithms is based on continuous relaxation of the discrete problem while using cotinuous optimization techniques or heuristics to optimize a matching score. Classic work include Relaxation Labeling~\cite{RL1,RL2,RL3,RL4} and Graduated Assignment~\cite{GA}. Generally, continuous relaxation based algorithms have lower computational costs than heurstic search based ones~\cite{GM30}.

In this paper, we propose a novel fast graph
matching algorithm called Fast Projected Fixed-Point (FastPFP), which is
capable of dealing with large graphs of over 1,000 nodes. By using a
new projected fixed-point method and a new partial doubly stochastic
projection, our algorithm has time complexity $O(n^3)$ per iteration
and space complexity $O(n^2)$.  In addition to its
scalability, our algorithm is easy to implement, robust, and able to match
undirected weighted attributed graphs of different sizes. We also proved the linear convergence of the new projected fixed-point algorithm, based on the theory of Convex Projection \cite{Projection1}\cite{Projection2}. To the best of our knowledge, our algorithm is the only iterative graph matching algorithm with linear convergence guarantee. We conduct
extensive experiments on benchmark datasets of large graphs. FastPFP demonstrated
better performance compared to previous state-of-the-art algorithms in most cases (only except for the graph isomorphism case, in which Umeyama's method is the fastest algorithm). In particular, in a PC, FastPFP is able to
match two graphs of 1,000 nodes within a few seconds.

The rest of this paper is organized as follows: In Section 2, we review previous work on continuous relaxation based graph matching algorithms. A simplied analysis of performance of different algorithms is given. In Section 3, we present our formulation of the graph matching problem. In Section 4, we introduce our FastPFP algorithm including the derivation, the new projection method, and convergence analysis. In Section 5, we show extensive experiments conducted on various benchmark datasets in comparison with previous state-of-the-art fast algorithms in large graph matching. In Section 6, we discuss FastPFP's comparison to FastGA, parameter sensitivity and limitations of our algorithm. Finally, we give concluding remarks in Section 7.

\section{Previous Work}
\begin{table*}
\caption{Characteristics of graph matching algorithms based on continuous relaxation}
\begin{center}
\begin{tabular}{c|c|c|c|c|c|c}
\hline
\multirow{2}{*}{Algorithm} & Time & Space & Able to Match Graphs  & Calling Hungarian method  & Convergence & Convergence \\
 &  Complexity & Complexity & of Different Sizes? & for each iteration? & Guarantee  & Rate \\
\hline
LP & $O(n^6)$ & $O(n^4)$ & No & No & N/A & N/A \\
RL & $O(n^4)$/iteration & $O(n^2)$ & Yes & No & Yes & Unknown\\
GA, POCS & $O(n^4)$/iteration & $O(n^4)$ & Yes & No & Yes & Unknown\\
RRWM & $O(n^4)$/iteration & $O(n^4)$ & Yes & No & No & N/A \\
IPFP & $O(n^4)$/iteration & $O(n^4)$ & Yes & Yes & Yes & Unknown\\
SM, SMAC & $O(n^4)$ & $O(n^4)$ & Yes & No & N/A & N/A \\
PGM & $O(n^4)$ & $O(n^2)$ & Yes & No & N/A & N/A \\
PATH & $O(n^3)$/iteration & $O(n^2)$ & Yes & Yes & Yes & Unknown\\
Umeyama & $O(n^3)$ & $O(n^2)$ & No & No & N/A & N/A \\
FastGA, PG & $O(n^3)$/iteration & $O(n^2)$ & Yes & No & Yes & Unknown\\
FastPFP & $O(n^3)$/iteration & $O(n^2)$ & Yes & No & Yes & Linear \\
\hline
\end{tabular}
\end{center}
\end{table*}
In this section, we review and analyze continuous relaxation based graph matching algorithms.
The main idea of continuous relaxation approach is to relax the original discrete optimization problem into a continuous one and solve it by continuous optimization techniques. For a simplified analysis, let us consider the performance of different algorithms based on continuous relaxation in matching two weighted graphs of $n$ nodes. Relaxation Labeling (RL)~\cite{RL1,RL2,RL3,RL4} is an iterative updating algorithm with the
assumption that the likelihood of two nodes to be
corresponded can be inferred from their neighboring nodes. The algorithm has time complexity $O(n^4)$ per iteration. Graduated Assignment (GA)~\cite{GA} is another iterative algorithm based on continuous relaxation. It enjoys a clear optimization objective function and is widely regarded as a state-of-the-art algorithm in the field. Several
techniques, including continuous relaxation, linear approximation,
softmax and Sinkhorn normalization, are used to optimize this
objective function. Compared to
RL, GA also has time complexity $O(n^4)$ per iteration but is more accurate and robust~\cite{GA}. Although the Linear Programming (LP) \cite{LP} approach is also based on linear approximation as GA, it has time complexity $O(n^6)$.

The recently proposed new graph matching algorithms include
Projections Onto Convex Sets (POCS)~\cite{POCS}, Spectral Matching
(SM)~\cite{SM}, Spectral Matching with Affine Constraint
(SMAC)~\cite{SMAC}, Integer Projected Fixed Point
(IPFP)~\cite{IPFP},
and Reweighted Random Walks Matching (RRWM)~\cite{RRWM}. They all have time complexity
$O(n^4)$  (per iteration), since they all require to construct and compute on a
compatibility matrix or an association graph, which is an $n^2\times
n^2$ matrix. Probabilistic Graph Matching (PGM)~\cite{PGM} requires the construction of
a marginalization matrix, which takes $O(n^4)$ operations as well.
Note that, although it is claimed that the
compatibility matrix can be very sparse so that sparse matrix
techniques can be used for efficient storage and computation, this
is not the case for either weighted graphs or densely connected
unweighted graphs. Path Following (PATH) algorithm \cite{PATH} has time complexity $O(n^3)$ per iteration. However, it calls the Hungarian method~\cite{Hungarian} many times for each iteration. Although the Hungarian method has time complexity $O(n^3)$, the constant factor inside the big-O notation is large. In a PC, it takes more than 500 seconds to apply Hungarian method to a $1000\times 1000$ non-sparse random matrix (See Fig. \ref{hungarian}).
Therefore, these methods do not have good
scalability with respect to graph size. Although some of the above algorithms have high matching accuracy, their high computational costs prohibit them from matching large graphs (e.g. graphs of 1,000 nodes) within a reasonable time.

\begin{figure}
\centering
\includegraphics[height=5cm,width=6cm]{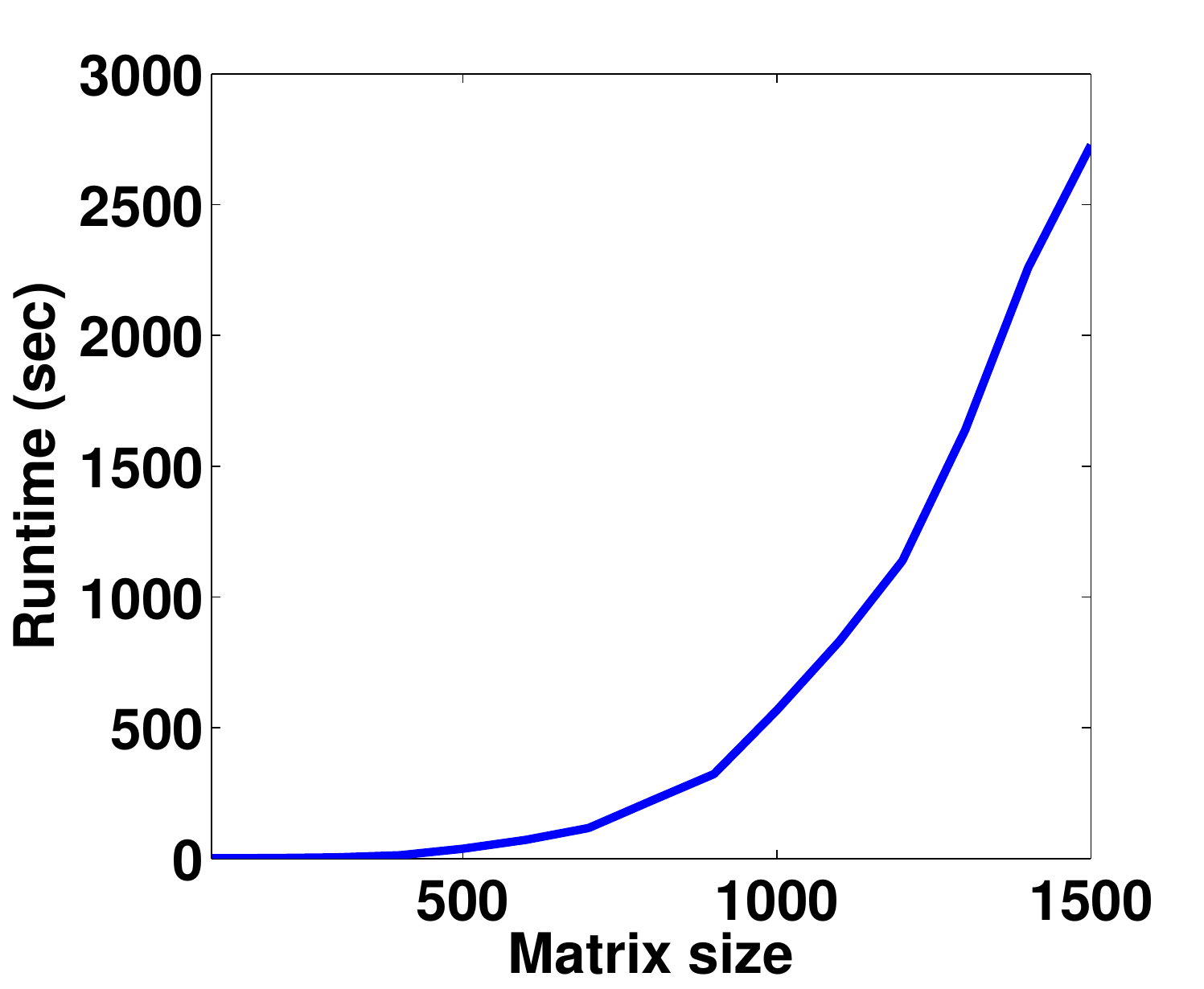}\label{hungarian}
\caption{Runtime of Hungarian method applied to random non-sparse matrices. Matrix size ranges from $100\times 100$ to $1500\times 1500$.}
\end{figure}

Distinct from the above algorithms, there are three fast graph matching algorithms, all having time complexity $O(n^3)$ (per iteration): Umeyama's methods~\cite{Umeyama}, FastGA~\cite{GA} and Projected Gradient~\cite{PG}. Umeyama's method uses eigendecomposition and Hungarian methods only once, both of which have time complexity $O(n^3)$. In \cite{PATH}, it was shown that Umeyama's method is more than an order of magnitude faster than PATH. Despite of its efficiency, Umeyama's method is unrobust, especially when the two graphs to be matched are far away from isomorphic (see Section \ref{sec:exp}). Moreover, it can only match graphs of equal sizes, which largely limits its use in practice. GA in general has time complexity $O(n^4)$ per iteration. However, by using a simple trick (see
Appendix), GA can be directly reduced to have time complexity $O(n^3)$ per iteration. We call the reduced GA as FastGA. Although IPFP can also be reduced to $O(n^3)$ per iteration using a similar trick, IPFP calls the Hungarian method for each iteration and it takes $5\sim10$ iterations for IPFP to converge \cite{IPFP}. Therefore, the reduced IPFP is still not suitable for large graph matching. Projected Gradient (PG) is a classic method to solve constrained optimization problems including graph matching. However, PG can get stuck in poor local optima easily in solving the graph matching problem, as we observed empirically (Section \ref{sec:exp}). The characteristics of different graph matching algorithms based on continuous relaxation are summarized in Table 1.

The proposed new projected fixed-point algorithm for graph matching is substantially different from the previous projected fixed-point method IPFP in three ways: (1) The projection of IPFP is onto a discrete domain, the space of partial permutation matrices, while the projection of FastPFP is onto a continuous domain, the space of partial doubly stochastic matrices, by using a new projection method \cite{DSN} (Section \ref{PDSN}). (2) IPFP requires the expensive Hungarian method for each iteration, making it unsuitable for large graph matching. On the other hand, each iteration of FastPFP only requires simple arithmetic operations (matrix addition and multiplication). (3) IPFP is proved to converge at finite steps \cite{IPFP} but the convergence rate is unknown. On the other hand, FastPFP is proved to converge at a linear rate.

\section{Problem formulation}\label{sec:Problem_formulation}
For two undirected graphs of size $n$ and $n'$, denote their adjacency
matrices (symmetric) by $A$ and $A'$ (binary valued for unweighted
graphs and real valued for weighted graphs) and attribute matrices
by $B$ and $B'$, respectively. Each row of $B$ and $B'$ is a
$k$-dimensional vector representing the attributes of a node. The
size of $A$, $A'$, $B$ and $B'$ are $n\times n$, $n'\times n'$,
$n\times k$ and $n'\times k$, respectively. The objective function of graph matching is
\begin{equation}\label{formulation1}
\min_X \frac{1}{2}\| A-XA'X^T \|^2_F + \lambda\| B -XB' \|^2_F,
\end{equation}
\begin{equation}\label{formulation2}
s.t. \quad X \mathbf{1}\leq\mathbf{1}, X^T\mathbf{1}=\mathbf{1}, X
\in \{0,1\}^{n\times n'},
\end{equation}
where $\|.\|_F$ is Frobenius matrix norm, $\lambda$ is a control
parameter, and $\mathbf{1}$ is a vector with all its elements equal
to one. In (\ref{formulation1}), the left term can be interpreted as dissimilarity between edges and the right term as dissimilarity between nodes. Constraints in (2) enforces that $X$ is a partial
permutation matrix. The constant $\frac{1}{2}$ is for convenience,
to be seen later. We assume without losing generality that $n\geq
n'$ in the paper. The minimization problem (\ref{formulation1})(\ref{formulation2}) is
equivalent to
\begin{equation}\label{QAP1}
\max_{X} \  \frac{1}{2}tr(X^{T}AXA') + \lambda tr(X^TK),
\end{equation}
\begin{equation}\label{QAP2}
s.t. \quad X\mathbf{1}\leq\mathbf{1}, X^T\mathbf{1}=\mathbf{1}, X
\in \{0,1\}^{n\times n'},
\end{equation}
where $tr$ denotes the matrix trace and $K$ denotes the ${n\times n'}$ matrix $BB'^T$ (See Appendix for derivation). This problem is
a Quadratic Assignment Problem (QAP) \cite{QAP} and is in general
NP-hard.

\section{Algorithm}\label{sec:FastPFP}
Our proposed algorithm, FastPFP relies on continuous relaxation of the original discrete
problem. The domain of the original discrete
problem, the space of partial permutation matrices
\[
\{X\  |\  X\textbf{1}\leq\textbf{1}, X^{T}\textbf{1}=\textbf{1}, X
\in \{0,1\}^{n\times n'} \},
\]
is relaxed to onto the space of partial doubly stochastic matrices
\[
\{X\  |\  X\textbf{1}\leq\textbf{1}, X^{T}\textbf{1}=\textbf{1}, X
\geq 0\}.
\]
A projected fixed-point method is then defined on the relaxed
space.
In the following subsections, we will give the derivation of our algorithm.

\subsection{Fast Projected Fixed-Point}
By relaxing the domain of the original QAP problem
(\ref{QAP1})(\ref{QAP2}) onto the space of partial doubly stochastic
matrices, we have
\begin{equation}
\max_{X} \  \frac{1}{2} tr(X^{T}AXA') + \lambda tr(X^TK),
\end{equation}
\begin{equation}
s.t. \quad X\textbf{1}\leq\textbf{1}, X^{T}\textbf{1}=\textbf{1}, X
\geq 0,
\end{equation}
where $X \geq 0$ denotes that all elements of $X$ are nonnegative.
To derive the algorithm, note that
for
\begin{equation}
f(X) = \frac{1}{2}tr(X^{T}AXA') + \lambda tr(X^TK),
\end{equation}
where $X \in R^{n\times n'}$, the gradient of $f(X)$ is
\begin{equation}
\nabla f(X) = AXA' + \lambda K.
\end{equation}
In this paper, we introduce a new projected fixed-point algorithm defined as
\begin{equation}\label{fastPFP}
X^{(t+1)} = (1-\alpha)X^{(t)} + \alpha P_d(\nabla f(X^{(t)})),
\end{equation}
\begin{equation}
s.t. \quad  \alpha\in [0,1],
\end{equation}
where $P_d(\cdot)$ is a partial doubly stochastic projection, defined as
\begin{equation}
P_d(X)=\arg\min_d \| X - d \|_{F},
\end{equation}
\begin{equation}
s.t. \quad d\textbf{1}\leq\textbf{1}, d^{T}\textbf{1}=\textbf{1},
d\geq 0,
\end{equation}
and $\alpha$ is the step size parameter. For $\alpha = 1$, the algorithm is a straightforward projected fixed-point method. For $0 < \alpha < 1$, the algorithm is a proportionally updated projected fixed-point method. The proportionality introduces more smoothness of the updating process, which could help stabilizing the algorithm (Section \ref{sec:para}). We call this new algorithm Fast Projected Fixed-Point (FastPFP) due to its linear convergence guarantee (Section \ref{sec:convergence}).
It also enjoys a property: if the initial state $X^{(0)}$ is a partial
doubly stochastic matrix, then each $X^{(t)}$ stays in the space of
partial doubly stochastic matrices.
This is because for two partial doubly stochastic matrices $X$ and $Y$, their
convex combination
\begin{equation}
Z = (1-\alpha) X + \alpha Y, \quad s.t. \quad  \alpha\in [0,1],
\end{equation}
is another partial doubly stochastic matrix.

\subsection{Interpretation and related work}
In (\ref{fastPFP}), let $\alpha = 1$ and replace $P_d(\cdot)$ with the projection onto the space of partial permutation matrices, defined as
\begin{equation}
P_{perm}(X) = \arg\min_{P} \| X - P \|_F^2,
\end{equation}
\begin{equation}
s.t. \quad P\textbf{1} \leq\textbf{1}, P^{T}\textbf{1}=\textbf{1},  P \in \{0,1\}^{n\times n'},
\end{equation}
which can be solved by the Hungarian method. Then (\ref{fastPFP}) becomes essentially a simplified IPFP
\begin{equation}\label{IPFP}
X^{(t+1)} =  P_{perm}(\nabla f(X^{(t)})).
\end{equation}
The objective function (\ref{QAP1}) is increased for each step of IPFP, as proved in \cite{IPFP}. Therefore, FastPFP can be interpreted as a softened version of IPFP.

On the other hand, the Projected Gradient (PG) algorithm is defined as
\begin{equation}\label{pg}
X^{(t+1)} = P_d( X^{(t)} + \alpha \nabla f(X^{(t)})).
\end{equation}
Although FastPFP and PG are algorithmically similar, in the graph matching problem, PG can get stuck in poor local optima easily. Moreover, PG usually runs significant slower than FastPFP. These experimental findings will be shown in detail in Section \ref{sec:exp}. And despite of the convergence guarantee of PG \cite{PG}, the convergence rate of the PG algorithm for solving the graph matching problem is unknown.

In the next subsection, we will
discuss how to obtain the solution of the projection $P_d(\cdot)$.

\subsection{Partial Doubly Stochastic Projection}\label{PDSN}
The partial doubly stochastic projection can be converted to doubly
stochastic projection by introducing slack variables. Let $Y$ be the
slacked matrix of $X$
\begin{equation}
Y=\begin{pmatrix}
X_{11} & X_{12} & ... & X_{1n'} & Y_{1(n'+1)} &...&Y_{1n}\\
X_{21} & X_{22} & ... & X_{2n'} & Y_{2(n'+1)} &...&Y_{2n}\\
... & ... & ... & ...& ... & \\
X_{n1} & X_{n2} & ... & X_{nn'} & Y_{n(n'+1)} &...&Y_{nn}\\
\end{pmatrix}.
\end{equation}
We write the above definition of $Y$ compactly as $Y_{1:n,1:n'}=X$,
where $Y_{1:n,1:n'}$ denotes the matrix formed by the first $n$ rows
and the first $n'$ columns of $Y$. To project a real nonnegative
matrix onto the space of doubly stochastic matrices, the Sinkhorn
method \cite{Sinkhorn} is usually used \cite{GA,PGM,RRWM}, which
normalizes each row and column of a matrix alternatively. The
objective of the normalization is to find a nearest doubly
stochastic matrix $D$ to a matrix $Y$ under the relative entropy
measure. Zass and Shashua proposed another doubly stochastic
projection \cite{DSN}, which  is originally to find a symmetric
nearest doubly stochastic matrix $D$ to $Y$ under the Frobenius
norm. However their method is also applicable to asymmetric doubly
stochastic matrices, if $Y$ is asymmetric. The projection of a
matrix $Y$ onto the space of doubly stochastic matrices under the
Frobenius norm is the solution to the following problem
\begin{equation}\label{DSP}
P_D(Y)=\arg\min_D \| Y - D \|_{F},
\end{equation}
\begin{equation}\label{DSPst}
s.t. \quad D\textbf{1}=\textbf{1}, D^{T}\textbf{1}=\textbf{1}, D
\geq 0.
\end{equation}
With this doubly stochastic projection, we are able to prove the linear convergence of FastPFP in the next subsection.

The doubly stochastic projection problem can be solved by successive projection. Define two
sub-problems (projections) of (\ref{DSP}-\ref{DSPst}) as
\begin{equation}\label{PP1}
P_1(Y)=\arg\min_D \| Y - D \|_{F}, \quad s.t. \quad
D\textbf{1}=\textbf{1}, D^{T}\textbf{1}=\textbf{1},
\end{equation}
\begin{equation}\label{PP2}
P_2(Y)=\arg\min_D \| Y - D \|_{F}, \quad s.t. \quad D \geq 0.
\end{equation}
Both $P_1$ and $P_2$ have a closed-form solution
\begin{equation}\label{P1}
P_1(Y)=Y+(\frac{1}{n}I+\frac{\mathbf{1}^{T}X\mathbf{1}}{n^2}I -
\frac{1}{n}Y ) \mathbf{11}^{T} - \frac{1}{n}\mathbf{11}^{T}Y,
\end{equation}
\begin{equation}\label{P2}
P_2(Y)=\frac{Y+|Y|}{2}.
\end{equation}
The derivation is left in Appendix.
The successive projection works as follows:
$P_D(Y)=...P_2P_1P_2P_1P_2P_1(Y)$. The von Neumann successive
projection lemma \cite{VonNeumann} guarantees the successive
projection converges to $P_D(Y)$. The solution to the partial doubly
stochastic projection under Fronbenius norm
is therefore $P_d(X)=P_D(Y)_{1:n,1:n'}$.

\subsection{Convergence Analysis}\label{sec:convergence}
Now we state the convergence theorem of FastPFP. The proof of the theorem relies on Lemma 1.
\begin{lemma}
Define $\mathbf{\Omega}=\{X|X\mathbf{1}\leq\mathbf{1},
X^{T}\mathbf{1}=\mathbf{1}, X \geq 0\}$. For two real matrices $X_1$
and $X_2$, $\|P_d(X_1)-P_d(X_2)\|_F \leq \|X_1-X_2\|_F$.
\end{lemma}
\begin{proof}
Since $\mathbf{\Omega}$ is a closed convex set, $P_d(\cdot)$ is a
nonexpansive projection, which means $\|X_1-X_2\|_F \geq
\|P_d(X_1)-P_d(X_2)\|_F$. See \cite{Projection2} for the proof of
nonexpansivity of projections onto convex sets.
\end{proof}
\begin{theorem}
Given real matrix $X^{(0)}$, $\alpha\in [0,1]$ and $\|A\otimes
A'\|_2<\epsilon$, the series
\begin{equation}\label{FastPFP}
X^{(t+1)} = (1-\alpha)X^{(t)} + \alpha P_d(AX^{(t)}A'+\lambda K),
\end{equation}
converges at rate $1-\alpha+\alpha\epsilon$.
\end{theorem}
\begin{proof}
Denote $vec(X)$, $vec(Y)$, $vec(K)$ and $A\otimes A'$ by
$\mathbf{x}$, $\mathbf{y}$, $\mathbf{k}$ and $\mathbf{A}$,
respectively, where $vec(\cdot)$ is the vectorization of a matrix
and $\otimes$ is the Kronecker product. Then $Y^{(t)}=AX^{(t)}A'+\lambda
K$ is equivalent to
$\mathbf{y}^{(t)}=\mathbf{A}\mathbf{x}^{(t)}+\lambda \mathbf{k}$.
Thus
\begin{align}
\|X^{(t+1)}-X^{(t)}\|_F &\leq(1-\alpha)\cdot\|X^{(t)}-X^{(t-1)}\|_F \\
&+\alpha\cdot\|P_d(AX^{(t)}A'+\lambda K) \\
&-P_d(AX^{(t-1)}A'+\lambda
K)\|_F.
\end{align}
Also
\begin{align}
&\|P_d(AX^{(t)}A'+\lambda K)-P_d(AX^{(t-1)}A'+\lambda K)\|_F \\
&\leq\|AX^{(t)}A'-AX^{(t-1)}A'\|_F =\|\mathbf{A}\mathbf{x}^{(t)}-\mathbf{A}\mathbf{x}^{(t-1)}\|_2 \\
&\leq\|\mathbf{A}\|_2\cdot\|\mathbf{x}^{(t)}-\mathbf{x}^{(t-1)}\|_2  \\
&<\epsilon\cdot\|X^{(t)}-X^{(t-1)}\|_F.
\end{align}
Thus
\begin{equation}
\frac{\|X^{(t+1)}-X^{(t)}\|_F}{\|X^{(t)}-X^{(t-1)}\|_F}<1-\alpha+\alpha\epsilon.
\end{equation}
\end{proof}
\textbf{Remarks.} In practice, we do not have to scale $A$ and $A'$
such that $\|A\otimes A'\|_2 < 1$. We can let $X=X/\max(X)$ at the
end of each iteration to prevent numerical instability.
\subsection{Discretization}
After the convergence of the projected fixed-point method, the resulting
matrix $X$ is discretized to obtain the partial permutation matrix
$P$. We use the same discretization method in \cite{SM}, which is a
greedy algorithm, instead of the expensive Hungarian method. This
greedy discretization algorithms works as follows: \\

\noindent \textbf{Step 1.} Initialize an $n\times n'$ zero-valued
matrix $P$  and a set $L$ containing all its
index $(i,j)$.\\
\textbf{Step 2.} Given a matrix $X$, find the index $(i*,j*)$ from
$L$ such that $X_{i*j*}=\arg \max_{(i,j)\in L}X_{ij}$. Set
$P_{i*j*}=1$. Remove all indices $(i,j)$ in $L$
that $i=i*$ or $j=j*$. \\
\textbf{Step 3.} Repeat step 2 until $L$ is empty. Return $P$. \\

\noindent The whole algorithm is summarized in Algorithm 1.  Note
that $n\geq n'$ is assumed.
\begin{algorithm}
\caption{Fast Projected Fixed-Point} \KwIn{$A, A', K$} \KwOut{$P$}
Initialize $X$ and $Y$
\\
\Repeat{$X$ converges}{$Y_{1:n,1:n'}=AXA'+\lambda K$ \\
\Repeat{$Y$
converges}{$Y=Y+(\frac{1}{n}I+\frac{\mathbf{1}^{T}Y\mathbf{1}}{n^2}I
- \frac{1}{n}Y ) \mathbf{11}^{T} - \frac{1}{n}\mathbf{11}^{T}Y $\\
$Y=\frac{Y+|Y|}{2}$ } $X=(1-\alpha)X + \alpha Y_{1:n,1:n'}$ \\
$X=X/\max(X)$ } Discretize $X$ to obtain $P$.
\end{algorithm}
The matching accuracy of FastPFP is sensitive to initialization. We
recommend the initialization $X=\mathbf{11}^{T}_{n\times n'}/nn',
Y=0_{n\times n}$, as we used in all our experiments. Step 9 is
needed for numerical stability. For $n=n'$, regardless of fast and
sparse matrix computation, Step 3 in Algorithm 1 requires $O(n^3)$
operations per iteration. Step 4 to Step 7 requires $O(n^2)$
operations per iteration. Step 8 requires $O(n^2)$ operations per
iteration. Thus, the algorithm has time complexity $O(n^3)$ per
iteration and space complexity $O(n^2)$.

\section{Experiments}\label{sec:exp}
\begin{figure*}
\centering
\subfigure[No noise (graph
isomorphism)]{\includegraphics[height=3.5cm,width=9cm]{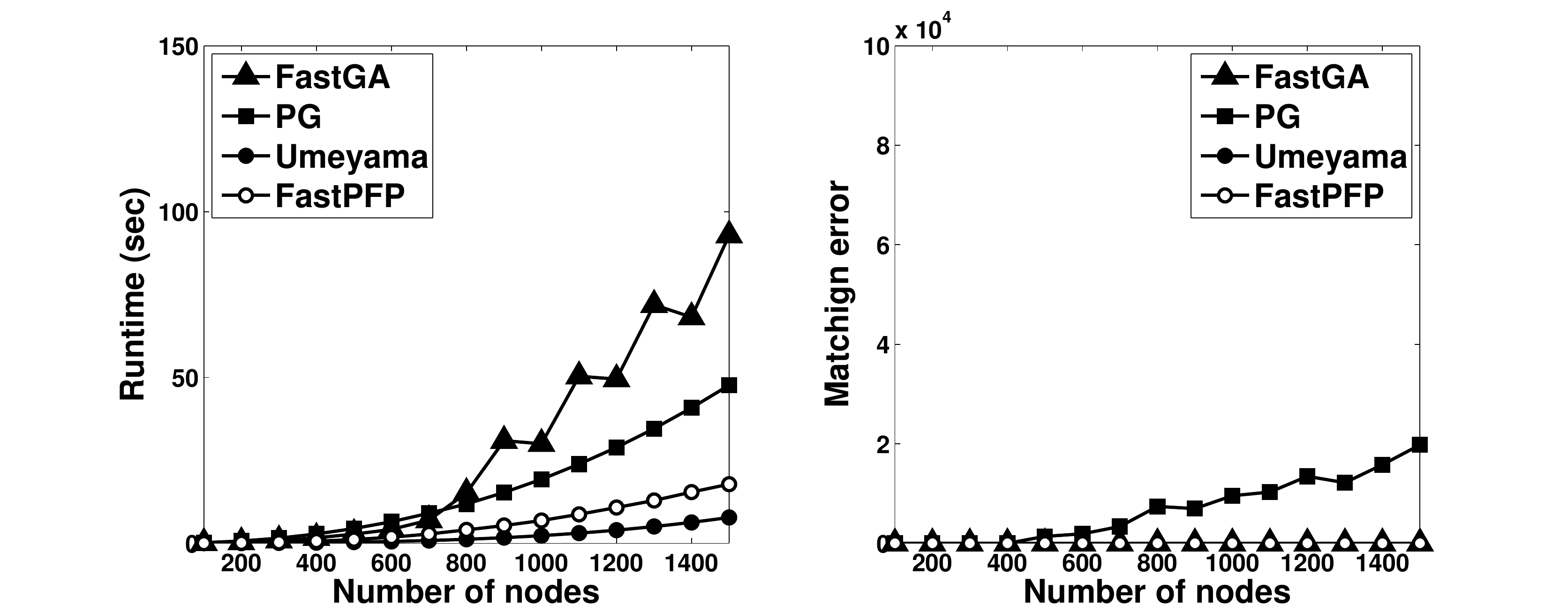}}
\subfigure[$n$ edges
edition]{\includegraphics[height=3.5cm,width=9cm]{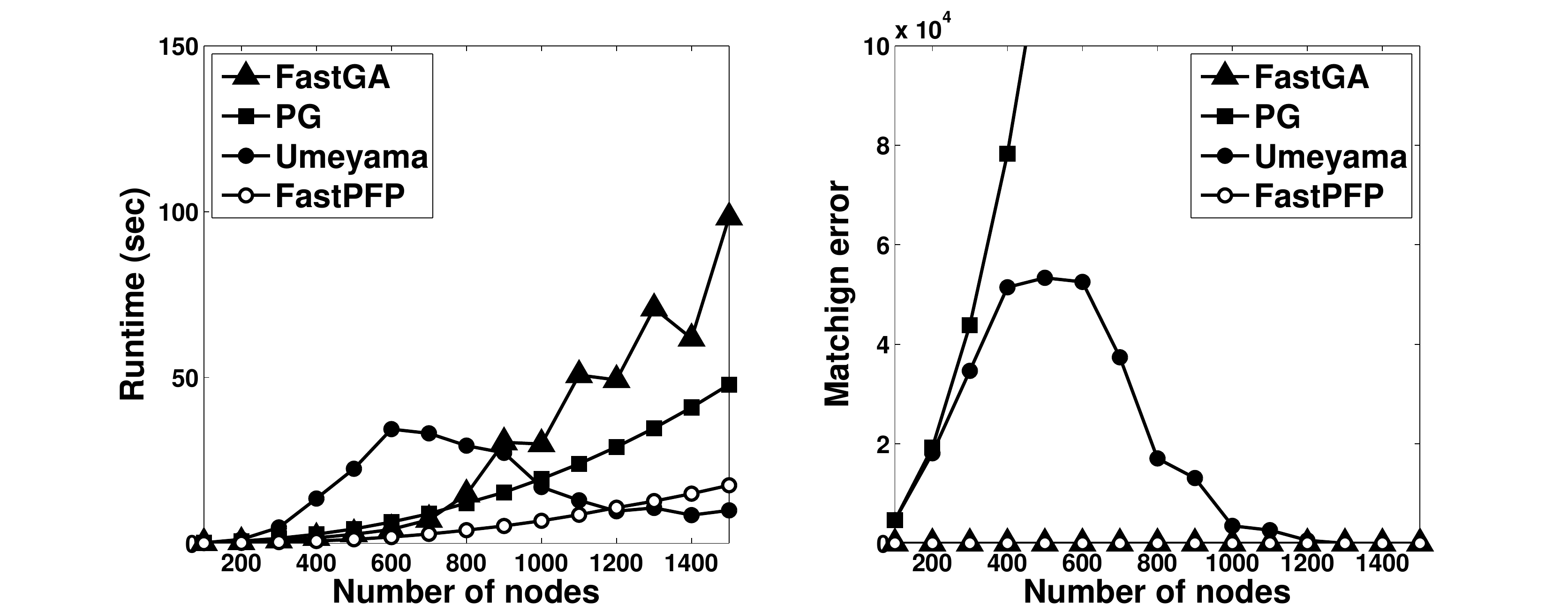}}
\subfigure[10\% nodes deletion (subgraph
isomorphism)]{\includegraphics[height=3.5cm,width=9cm]{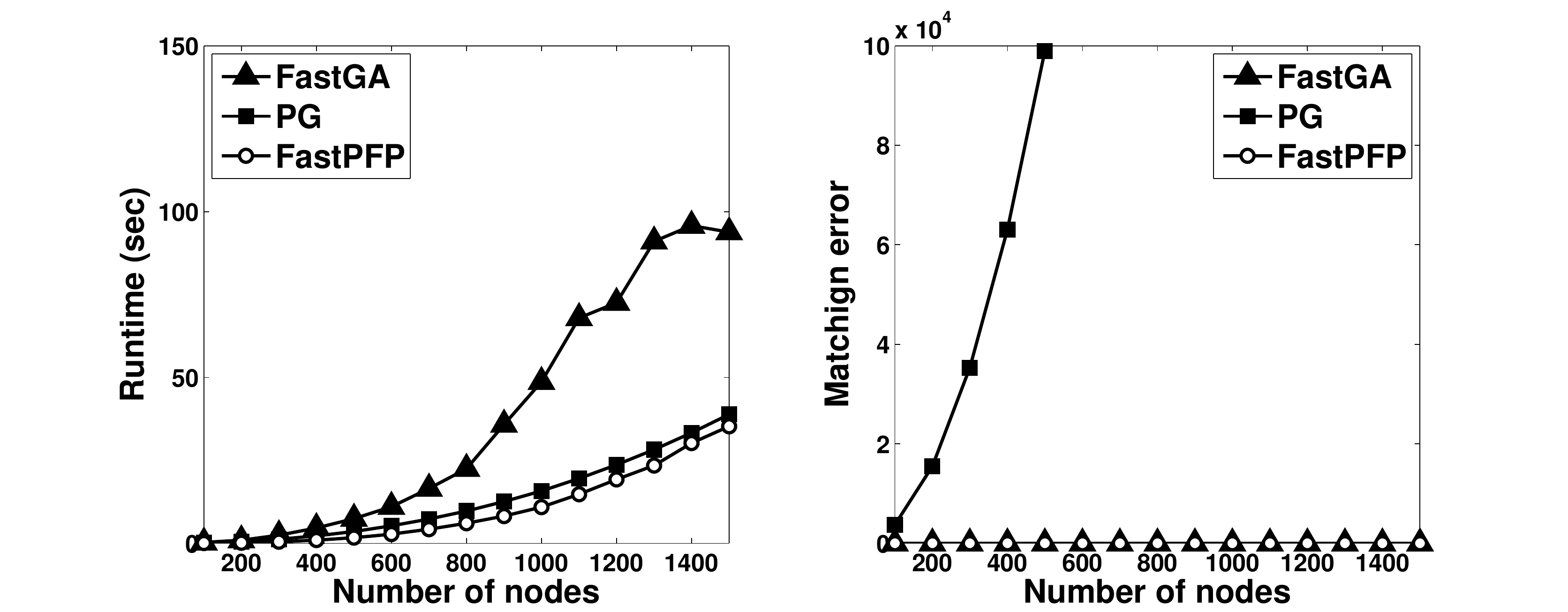}}
\subfigure[10\% nodes deletion
 + $n$ edges
edition]{\includegraphics[height=3.5cm,width=9cm]{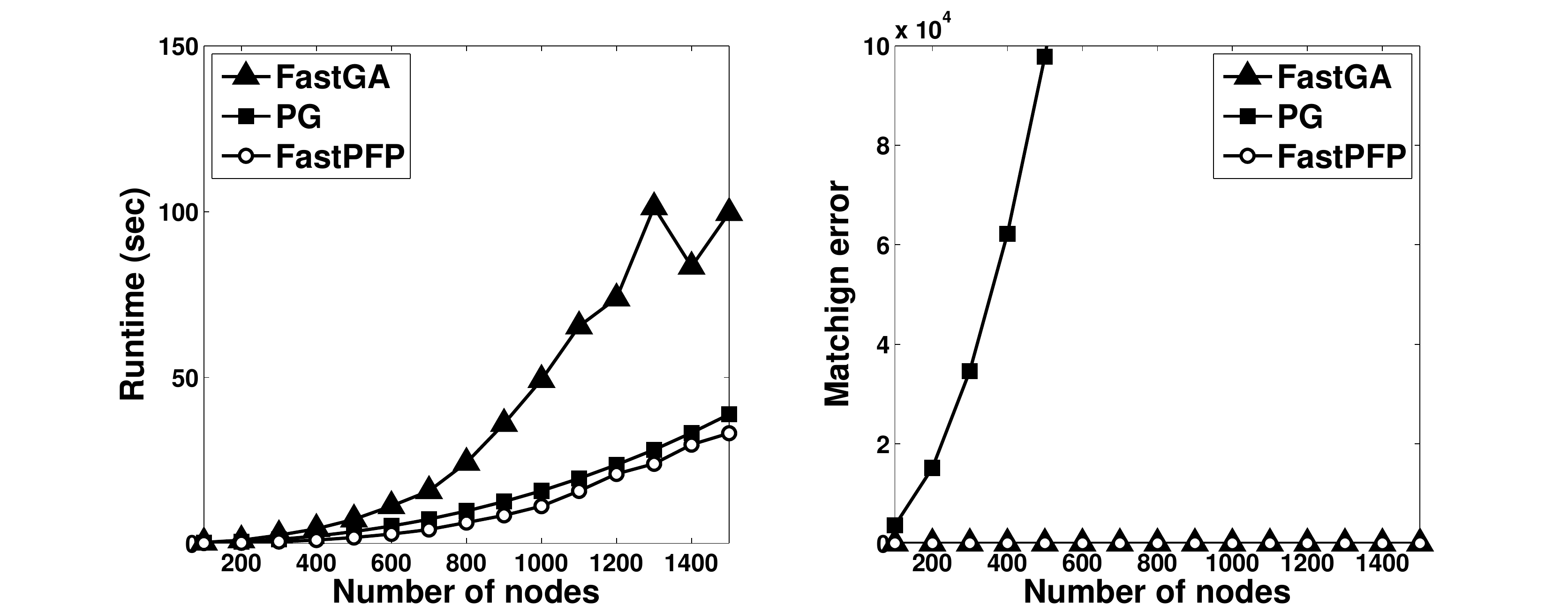}}
\caption{Synthetic random graph matching} \label{f1}
\end{figure*}

In this section, we evaluate the performance in large graph matching of four fast algorithms: FastPFP, Projected Gradient (PG) [\ref{pg}], FastGA and Umeyama's method. All of them have time complexity $O(n^3)$ (per iteration). Some other graph matching algorithms were not compared due to their inability to match large graphs (e.g. graphs of 1000 nodes) in a reasonable time.

For FastGA, FastPFP, and PG, the same greedy
discretization described in the above section was used. The
convergence criteria are: $\max(|X^{(t+1)}-X^{(t)}|)<\epsilon_1$ or
more than $I_0$ iterations for the first loop (projected fixed-point updating for FastPFP or soft-assignment for FastGA) and
$\max(|Y^{(t+1)}-Y^{(t)}|)<\epsilon_2$ or more than $I_1$ iterations
for the second loop (doubly stochastic projection or normalization).
For the fairness of comparison, we tuned and fixed the parameters of each
algorithm throughout all our experiments. The parameters were
well-tuned to achieve the best performance taking both accuracy and
efficiency into consideration. For FastPFP, we fixed $\alpha=0.5$ in all experiments. The sensitivity study of this parameter will be given in the next section. Other parameters were also tried to
validate the experimental results. All our experiments were
implemented in MATLAB in a 3 GHz Intel Core2 PC.
\\\textbf{Remarks.} (1) In some datasets, the experiments of Umeyama's method were not conducted due to its
inability of matching graphs of different sizes. (2) In the weighted graphs matching experiments, graphs are fully connected. Each node represents a keypoint and each edge represents Euclidean distance between two keypoints. We tried Delaunay triangularization to sparsify the graphs. But the matching quality was poor after applying the sparsification. Therefore, no sparsification was used in the experiments, for simplicity and better matching results.

\subsection{Synthetic Random Graphs}
In this set of experiments, unweighted random graphs were generated
uniformly with $50\%$ connectivity. The size of the graphs ranges
from 100 to 1,500. Four different types of graph matching were
tested: two isomorphic graphs with (a) no noise (graph isomorphism),
(b) $n$ edges edition (edge edition means if $A_{i,j}=1$, set
$A_{i,j}=0$ and vice versa), (c) $10\%$ nodes deletion (subgraph
isomorphism), (d) $n$ edges edition + $10\%$ nodes deletion. The
runtime and matching error (measure by $\| A-XA'X^T \|^2_F$ - $\| A-X_{GT}A'X_{GT}^T \|^2_F$, where $X_{GT}$ is the ground truth matching matrix) were
recorded, as shown in Fig. \ref{f1}.

As observed from Fig. \ref{f1}, FastPFP and FastGA were able to find the
global optimal solutions in all the above graph matching tasks. But FastPFP is about $3\sim5$
times faster than FastGA. PG were stuck in poor local minimum in most cases.
Umeyama's method achieved better performance than other algorithms for the isomorphic graph matching in terms of both speed and accuracy.
However, its performance was poor in matching non-isomorphic graphs. Roughly
speaking, the more isomorphic of two graphs, the better matching
results Umeyama's method had. In the left part of Fig. \ref{f1}(b) ,
the matching error increases with the number of nodes from 100 to
500. This is natural since the numeric range of the matching error
increases with the number of nodes given the same percent of
connectivity. From 600 nodes on, the matching error decreases with
the number of nodes. This is because the graphs are closer to
isomorphic as the number of nodes increases since there are only $n$
edges edition for total $O(n^2)$ edges.

\begin{figure*}
\centering
\subfigure[SIFT features (820 nodes vs. 772 nodes)]{\includegraphics[height=3.5cm,width=8.8cm]{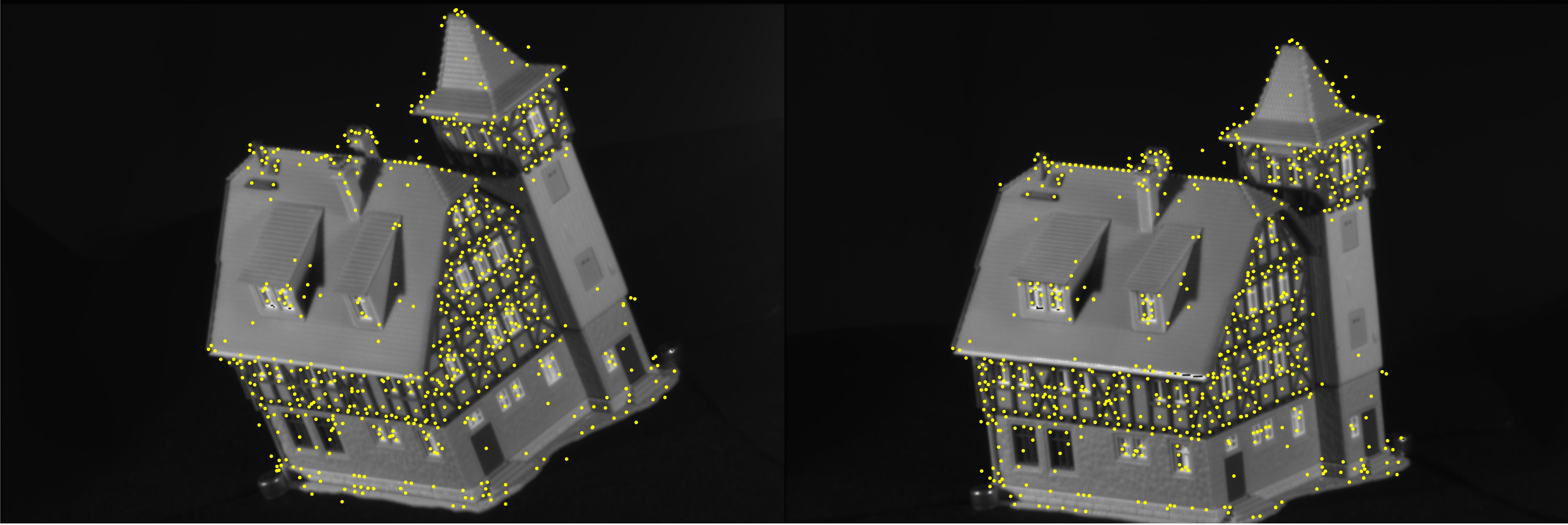}}
\subfigure[From top to bottom: PG, FastGA and FastPFP]{\includegraphics[height=7cm,width=\textwidth]{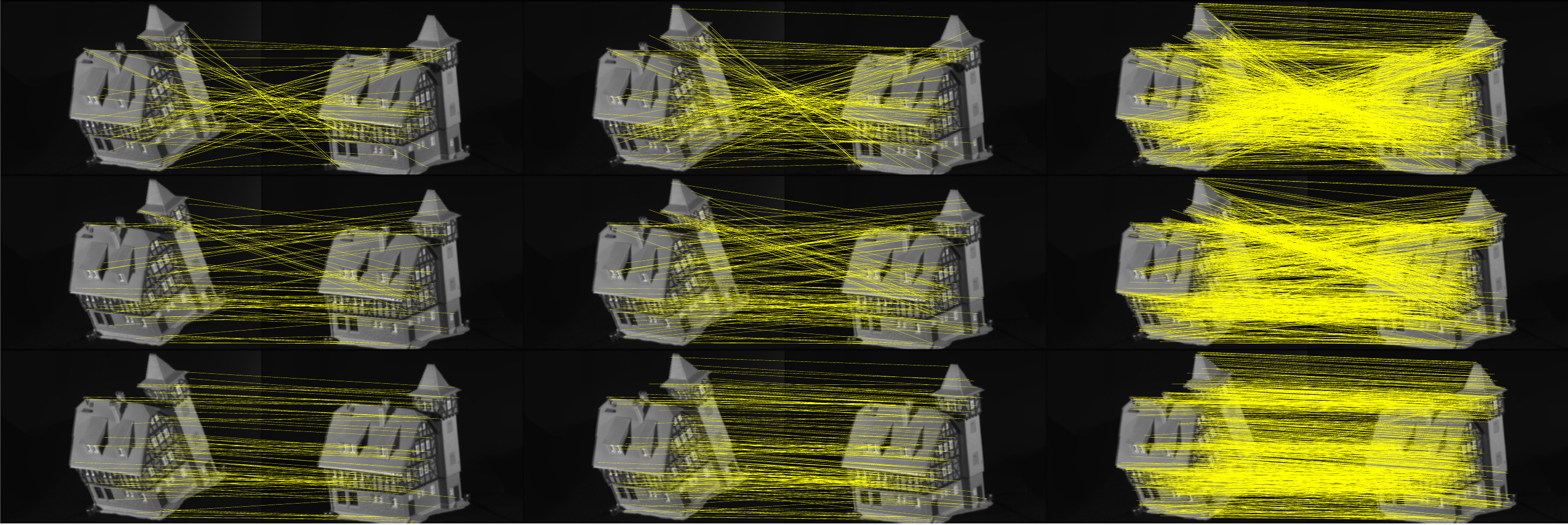}}
\subfigure[SIFT features (820 nodes vs. 772 nodes)]{\includegraphics[height=3.5cm,width=8.8cm]{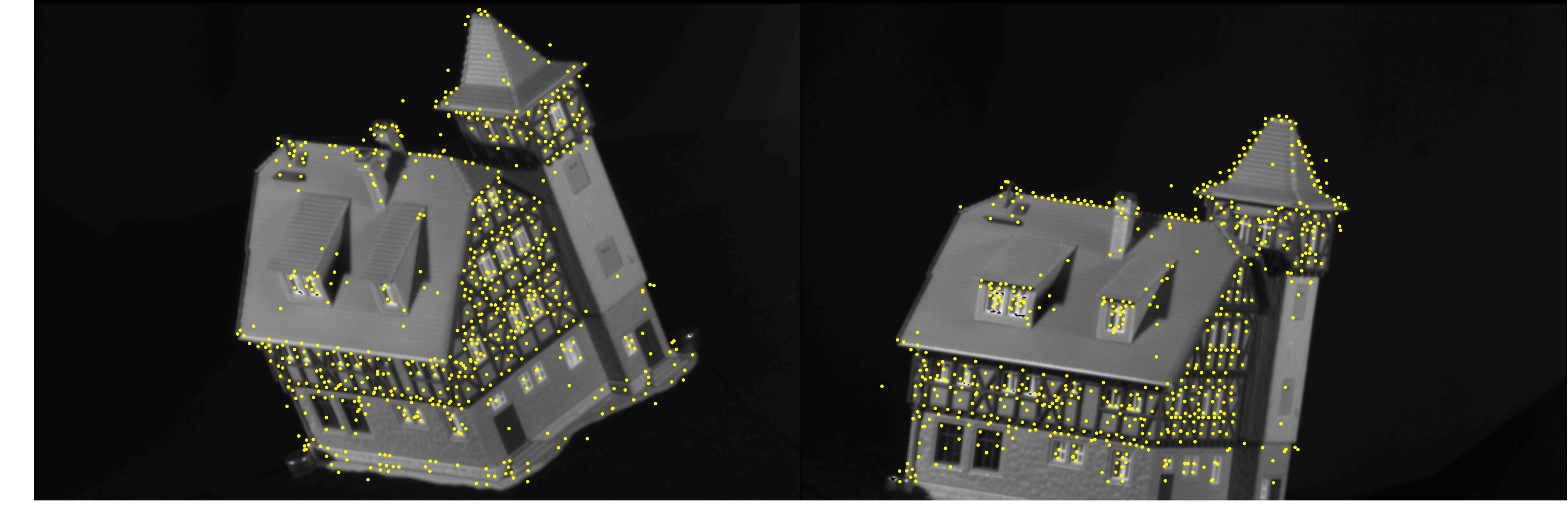}}
\subfigure[From top to bottom: PG, FastGA and FastPFP]{\includegraphics[height=7cm,width=\textwidth]{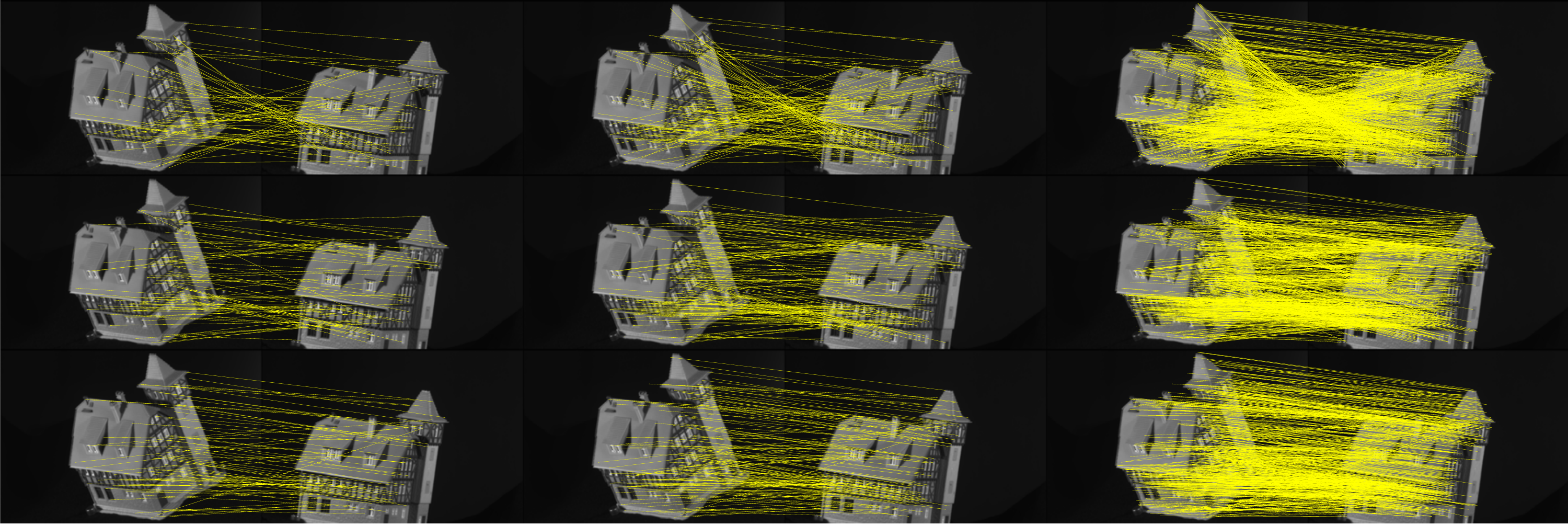}}
\caption{Graphs from CMU house sequence matching. For (c)(f), Left column: 10\% matching displayed. Middle column: 20\%
matching displayed. Right column: 100\% matching displayed.}
\label{CMU}
\end{figure*}
\begin{figure*}
\centering
\includegraphics[height=6cm,width=15cm]{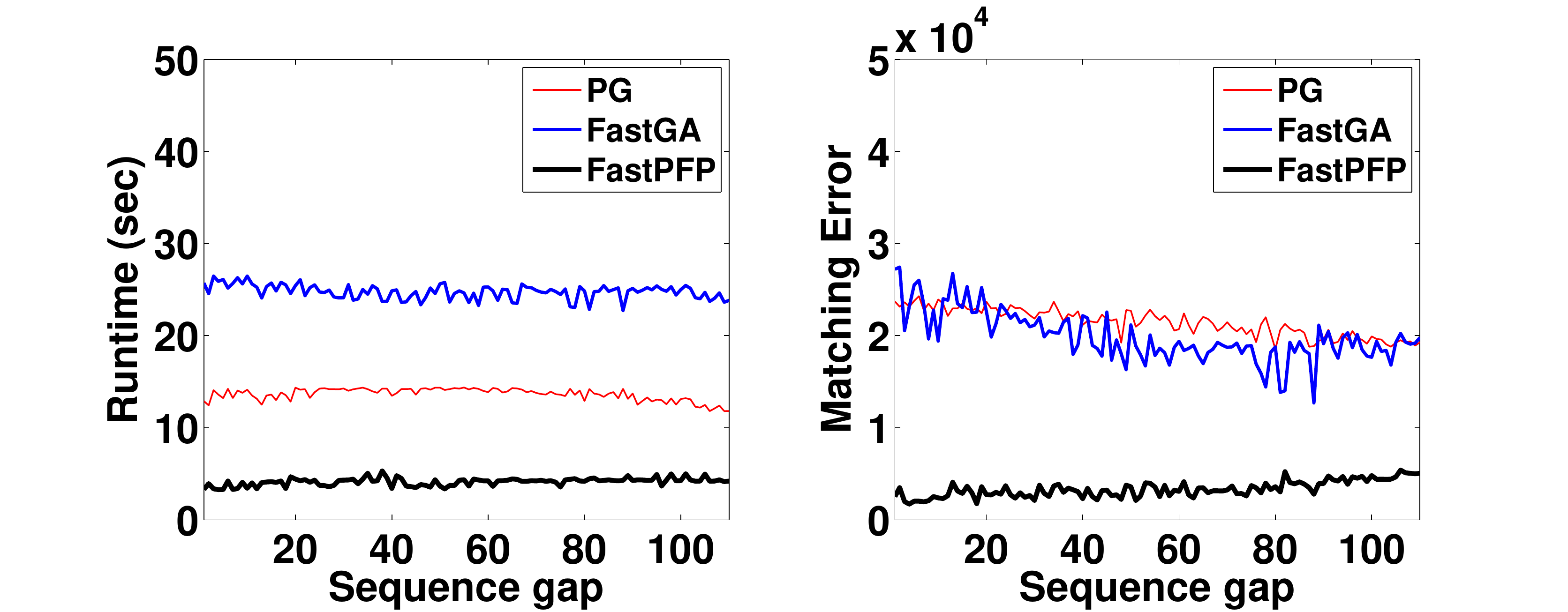}
\caption{CMU house sequence} \label{CMU_plots}
\end{figure*}
\begin{figure*}
\centering
\subfigure[SIFT features (1438 nodes vs. 1443 nodes)]{\includegraphics[height=3cm,width=8cm]{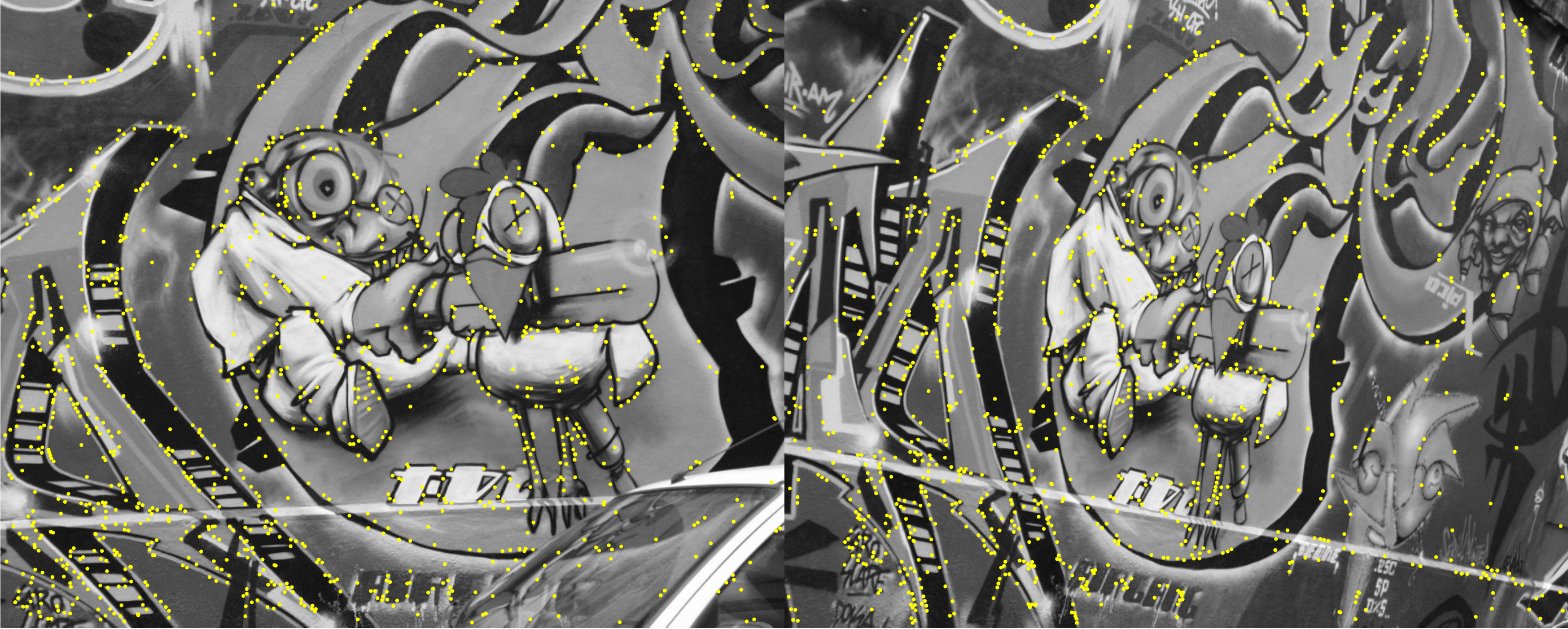}}
\subfigure[From top to bottom: PG, FastGA and FastPFP]{\includegraphics[height=7cm,width=\textwidth]{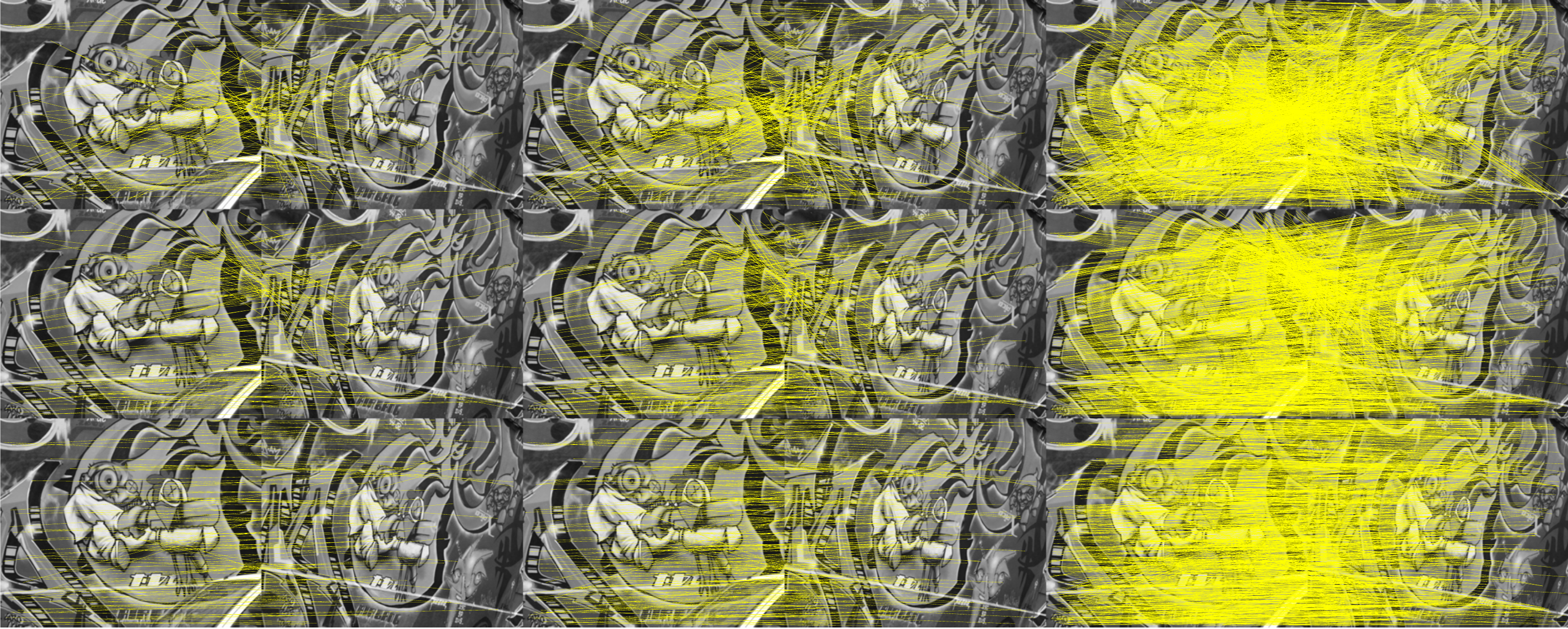}}
\subfigure[SIFT features (961 nodes vs. 960 nodes)]{\includegraphics[height=3cm,width=8cm]{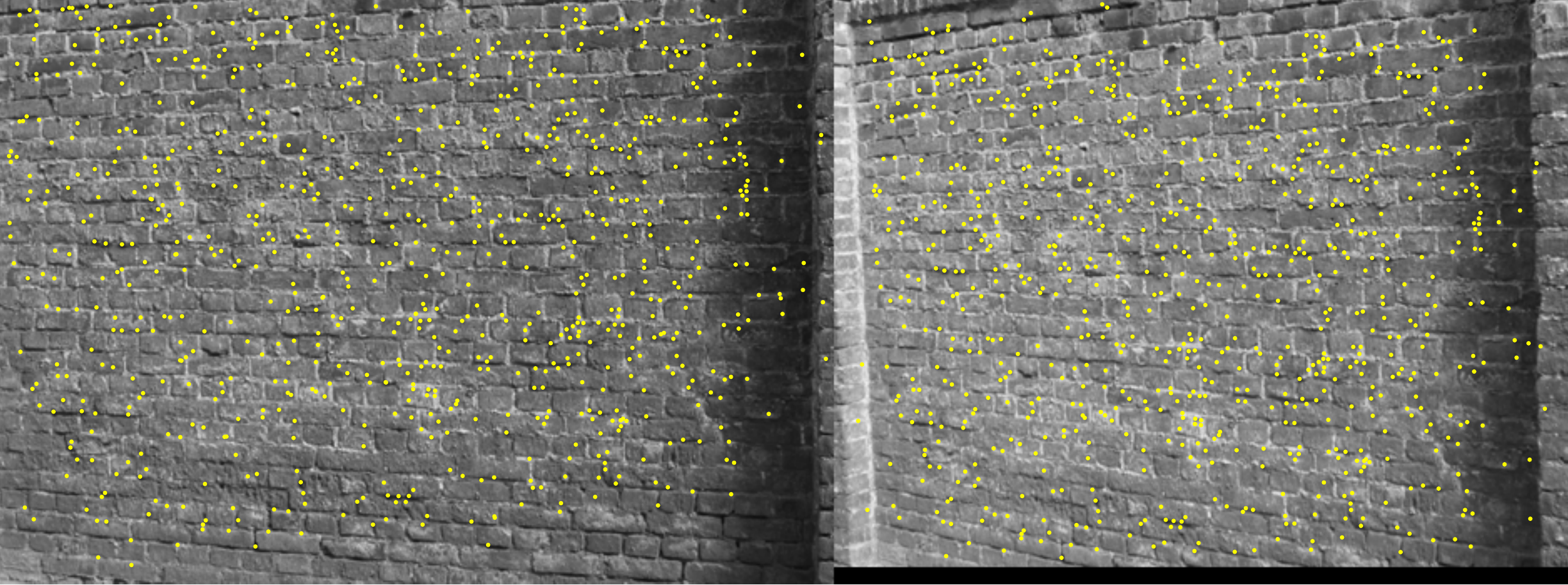}}
\subfigure[From top to bottom: PG, FastGA and FastPFP]{\includegraphics[height=7cm,width=\textwidth]{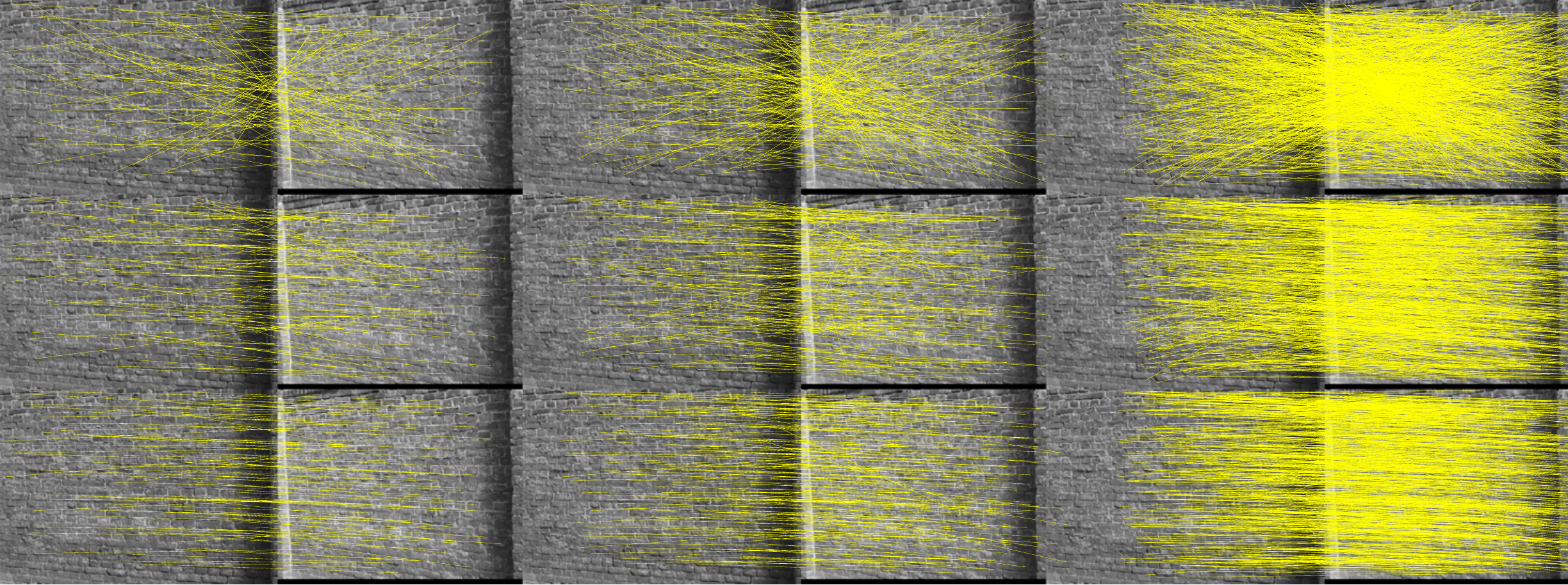}}
\caption{Graphs from real images matching. For (c)(f), left column: 10\% matching displayed. Middle column: 20\%
matching displayed. Right column: 100\% matching displayed.}
\label{fig:example}
\end{figure*}
\begin{figure*}
\centering
\subfigure[SIFT features (393 nodes vs. 391 nodes)]{\includegraphics[height=3cm,width=8cm]{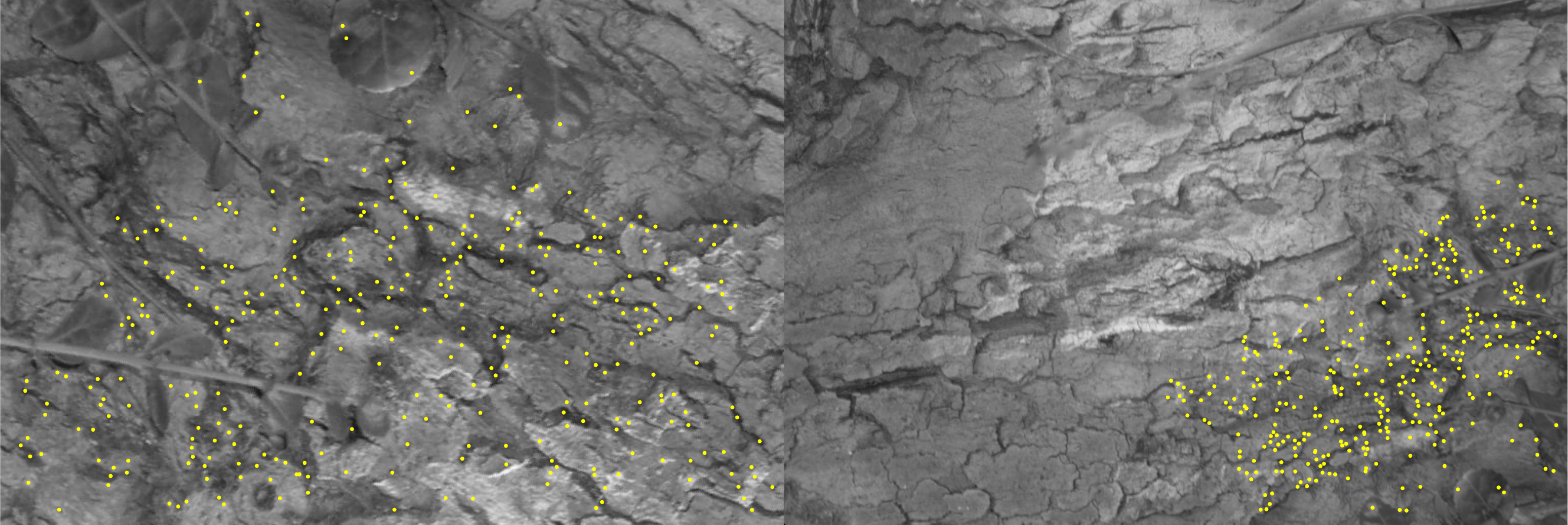}}
\subfigure[From top to bottom: PG, FastGA and FastPFP]{\includegraphics[height=7cm,width=\textwidth]{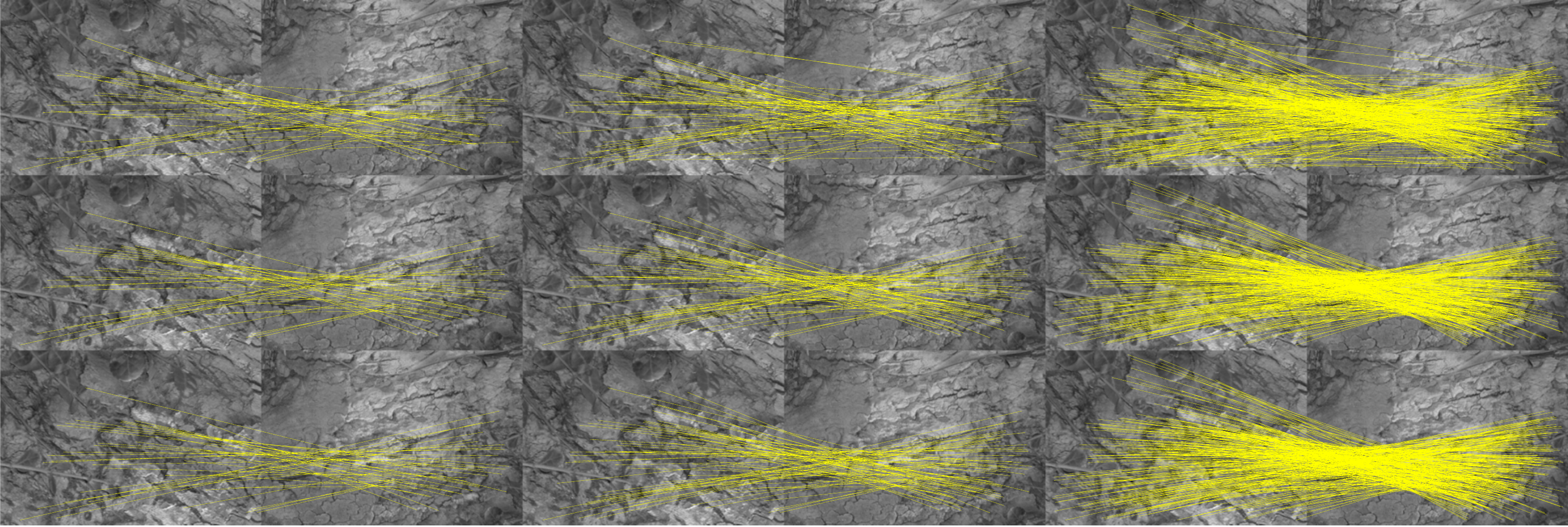}}
\subfigure[SIFT features (542 nodes vs. 537 nodes)]{\includegraphics[height=3cm,width=8cm]{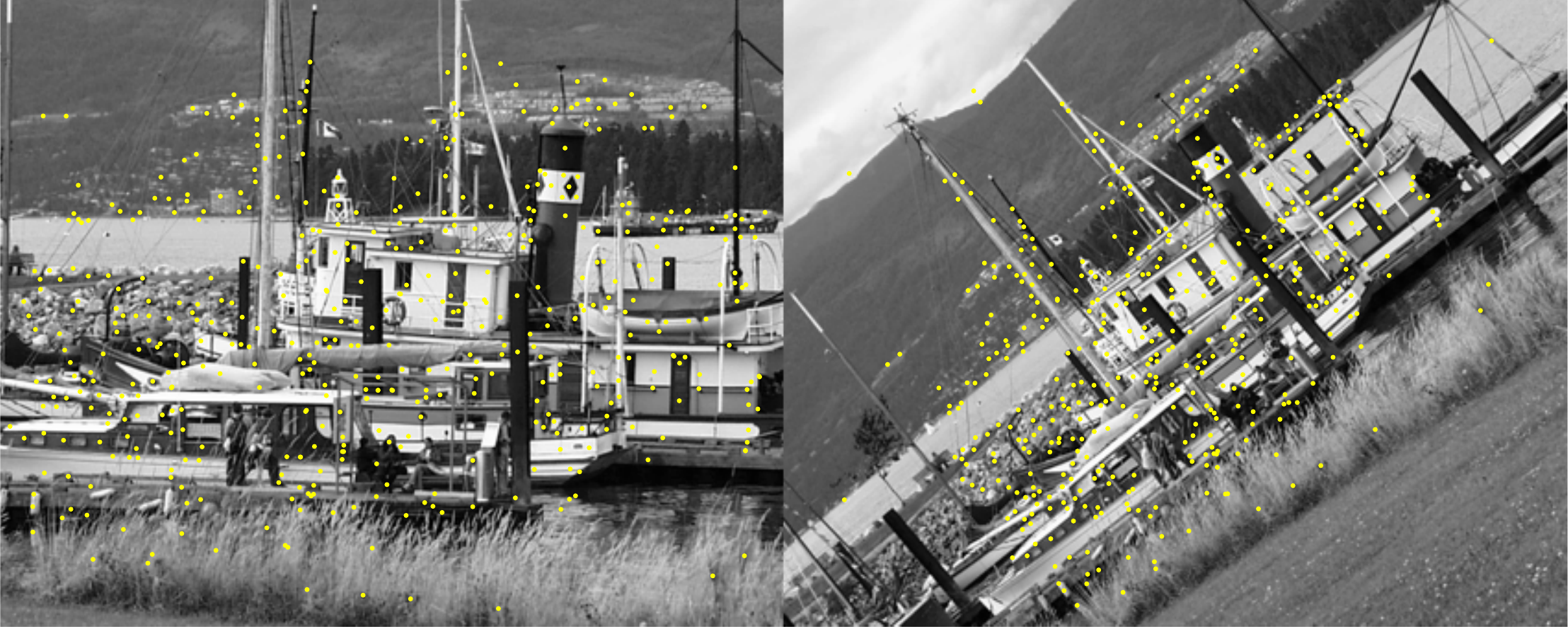}}
\subfigure[From top to bottom: PG, FastGA and FastPFP]{\includegraphics[height=7cm,width=\textwidth]{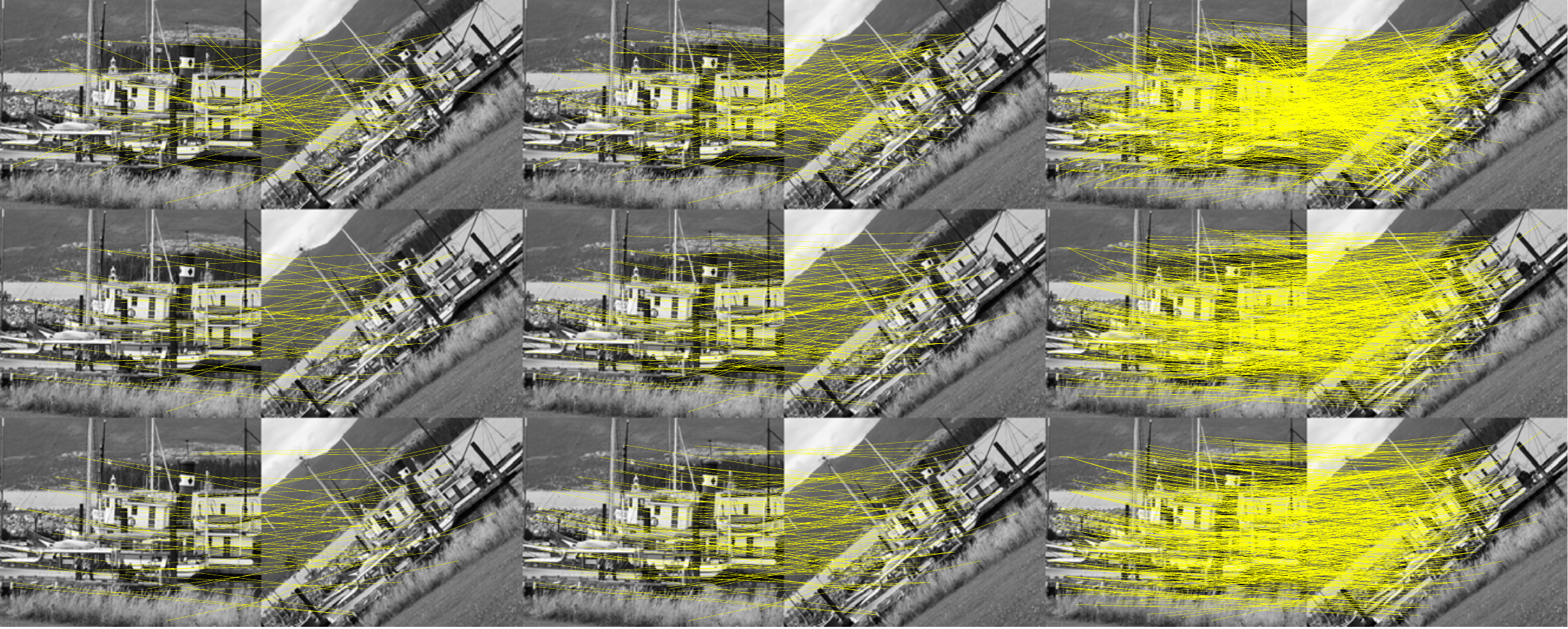}}
\caption{Graphs from real images matching. For (c)(f), left column: 10\% matching displayed. Middle column: 20\%
matching displayed. Right column: 100\% matching displayed.}
\label{fig:example}
\end{figure*}

\subsection{Graphs From Image Sequences}
In this set of experiments, weighted graphs were constructed from
the CMU House sequence images\footnote{\url{http://vasc.ri.cmu.edu/idb/html/motion/}},
which has been wildly used for testing graph matching algorithms. In
the previous literature, only small graphs (30 nodes) were extracted
from the images for matching. To demonstrate scalability of FastPFP,
we applied SIFT feature descriptor \cite{SIFT} to all CMU House
sequence images. Roughly 800 keypoints were extracted for each image
and the number of keypoints varies from image to image. All keypoints
in an image were used for the construction of a graph. For each
image, a graph whose nodes represent keypoints and edges (matrix $A$)
represent Euclidean distance between keypoints in that image was
constructed. For a visualization, we show in Fig. \ref{CMU}
feature correspondence for image 0 vs. image 110. Results of PG, FastGA
and FastPFP were shown.

We matched the first image and subsequent images, i.e. image $0$ vs.
image 1, image 0 vs. image 2, ..., and image 0 vs. image 110. The
runtime and matching error (measured as $\| A-XA'X^T \|_F^2$) were recorded for different algorithms.
The results are shown in Fig. \ref{CMU_plots}. FastPFP is about 6
times faster than FastGA, with even much lower matching error.

\subsection{Graphs From Real Images}
In this set of experiments, attributed and weighted graphs were constructed from
real images, chosen from the four image sets: Graffiti, Wall, Bark
and
Boat\footnote{\url{http://www.robots.ox.ac.uk/~vgg/research/affine/}}.
The graphs were constructed as follows. We first applied SIFT
descriptor to extract 128-dimensional features for each image. Then a simple feature selection was
performed as follows: a feature of an image was selected if it has
the similarity (inner product) to all features of its counterpart
image above a threshold. From the selected features, a graph was
constructed. Each node represents a keypoint, each node attribute represents a SIFT 128-dimensional feature vector, and each edge represents the
Euclidean distance between nodes. Each element $K_{ij}$ of $K$ is set to be the inner product of two SIFT feature vectors $i$ and $j$.
See Fig. \ref{fig:example} for visualization of the matching results for PG, FastGA
and FastPFP. The experimental results are listed in Table~\ref{real_runtime} and \ref{real_error}. The
runtime and matching error (measured as $\frac{1}{2}\| A-XA'X^T \|^2_F + \lambda\| B - XB' \|^2_F$) were recorded. $\lambda=1$ is used. The sensitivity study of parameter $\lambda$ is given in Section \ref{sec:para}.
Again, FastPFP demonstrated the best performance: it is about $3\sim6$
times faster than FastGA while usually achieving a much lower matching error.
\begin{table}
\begin{center}
\caption{Graphs from real images: runtime (sec)}\label{real_runtime}
\begin{tabular}{c|c|c|c}
\hline
Graph pair & PG & FastGA & FastPFP\\
\hline
Graffiti & 51.60 & 87.89 & $\mathbf{13.52}$ \\
Wall &19.94 & 32.25  & $\mathbf{4.92}$\\
Bark &2.84 & 2.72 & $\mathbf{0.52}$ \\
Boat &5.59 & 9.06  & $\mathbf{1.36}$ \\
\hline
\end{tabular}
\caption{Graphs from real images: matching error}\label{real_error}
\begin{tabular}{c|c|c|c}
\hline
Graph pair & PG & FastGA & FastPFP \\
\hline
Graffiti & 2.75$\times 10^4$ & 3.98$\times 10^4$ & $\mathbf{5.92\times 10^3}$\\
Wall & 1.31$\times 10^4$ & 1.01$\times 10^4$ & $\mathbf{2.39\times10^3}$ \\
Bark & 3.07$\times10^3$ & $\mathbf{2.21\times10^3}$ &  $\mathbf{2.21\times10^3}$ \\
Boat & 4.77$\times10^3$ & 3.18$\times10^3$ & $\mathbf{1.80\times10^3}$ \\
\hline
\end{tabular}
\end{center}
\end{table}

\subsection{Graphs From 3D Points}
\begin{figure*}
\centering
\subfigure[Face Pair 1 and Face Pair 2 (both 392 nodes vs. 392 nodes)]{
\includegraphics[height=4cm,width=7cm]{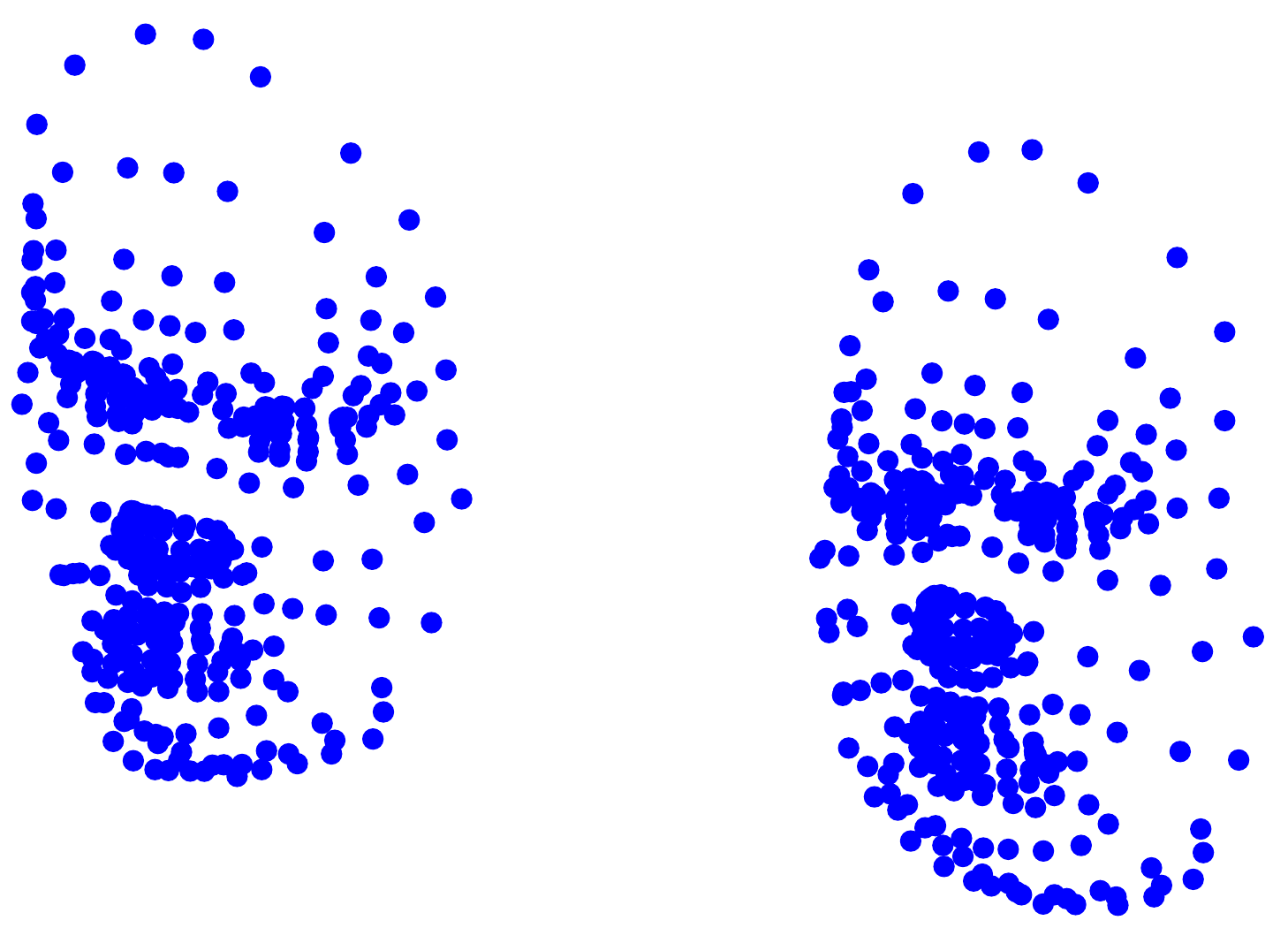}
\includegraphics[height=4cm,width=7cm]{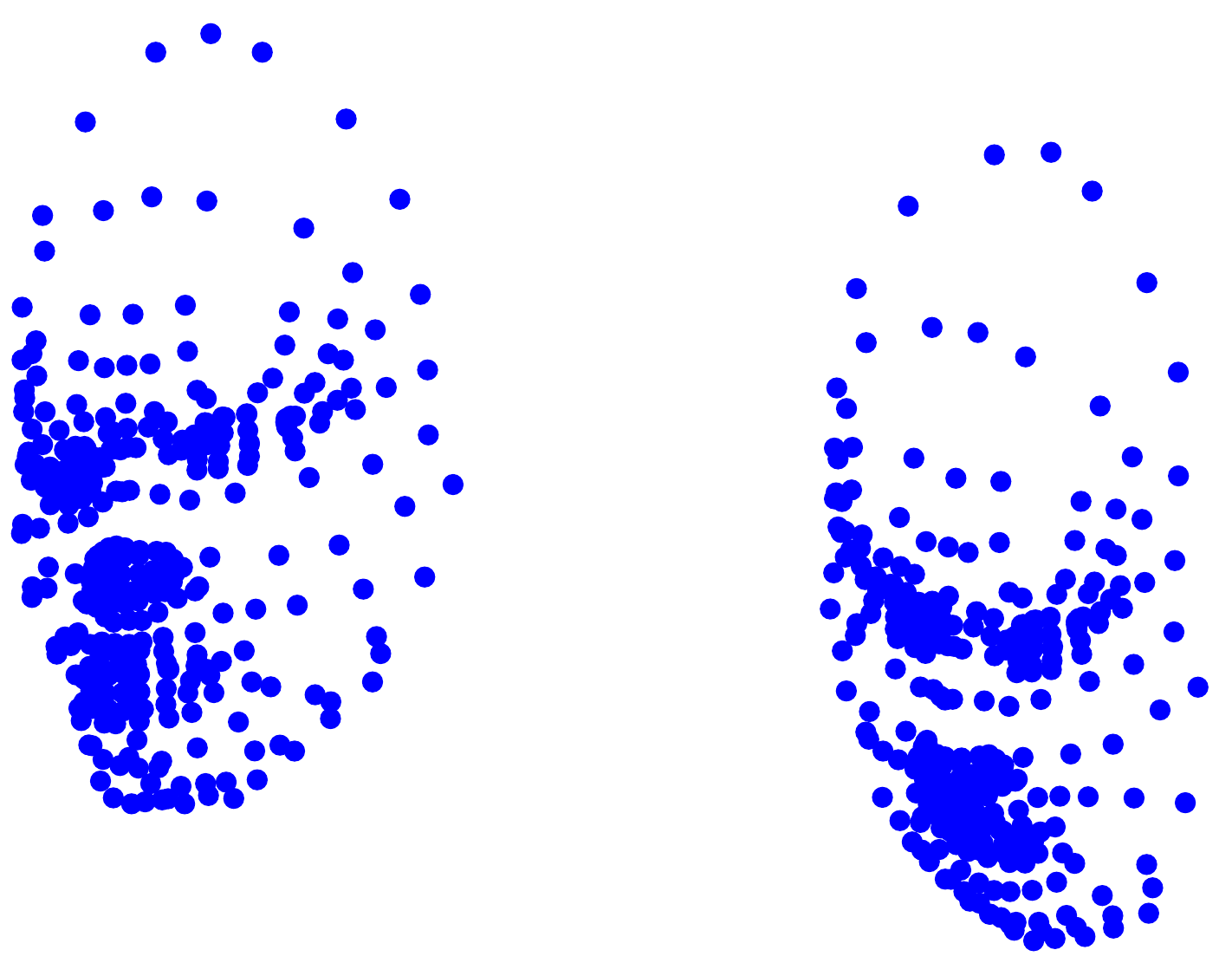}}
\subfigure[PG]{
\includegraphics[height=4cm,width=7cm]{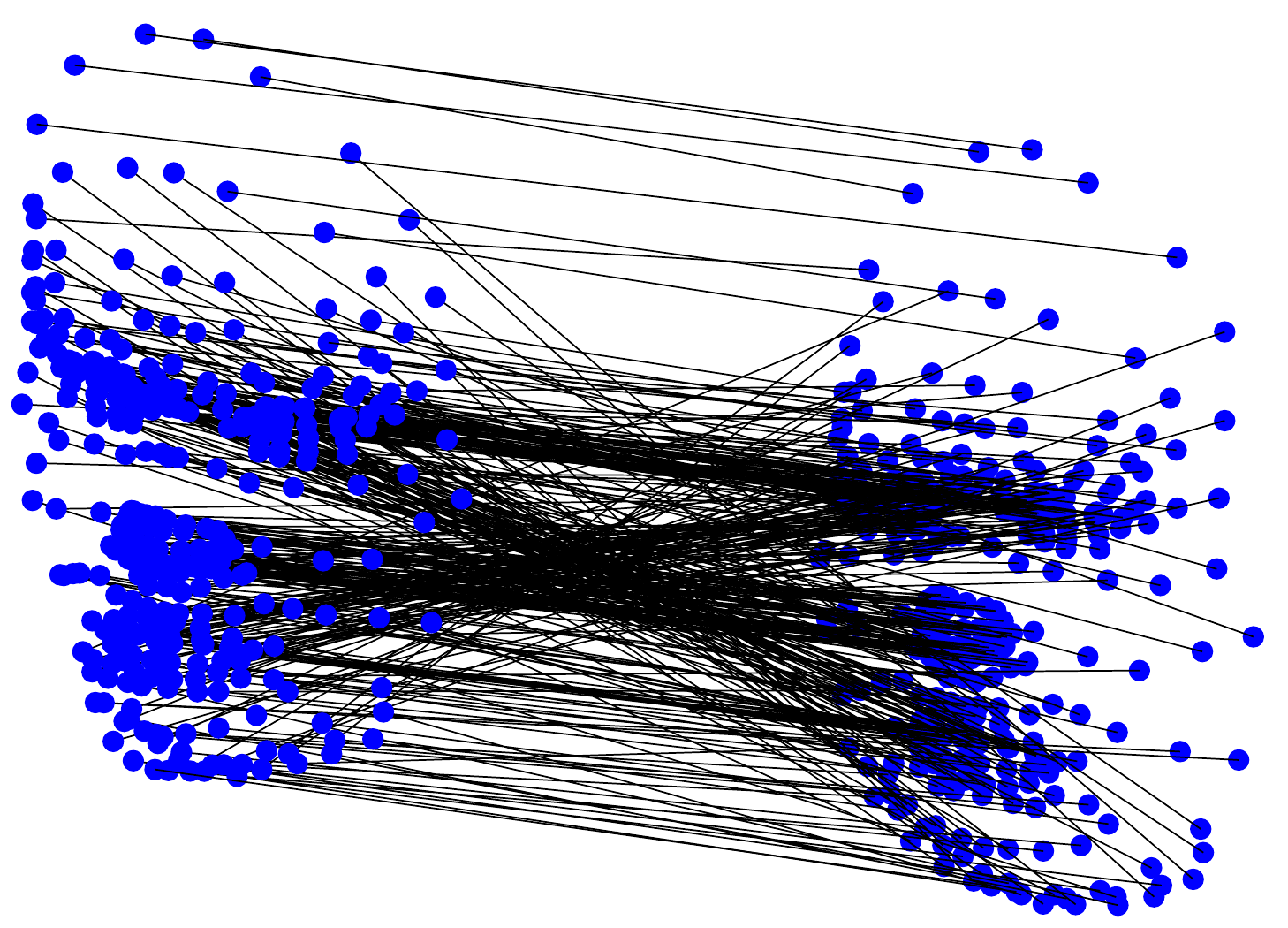}
\includegraphics[height=4cm,width=7cm]{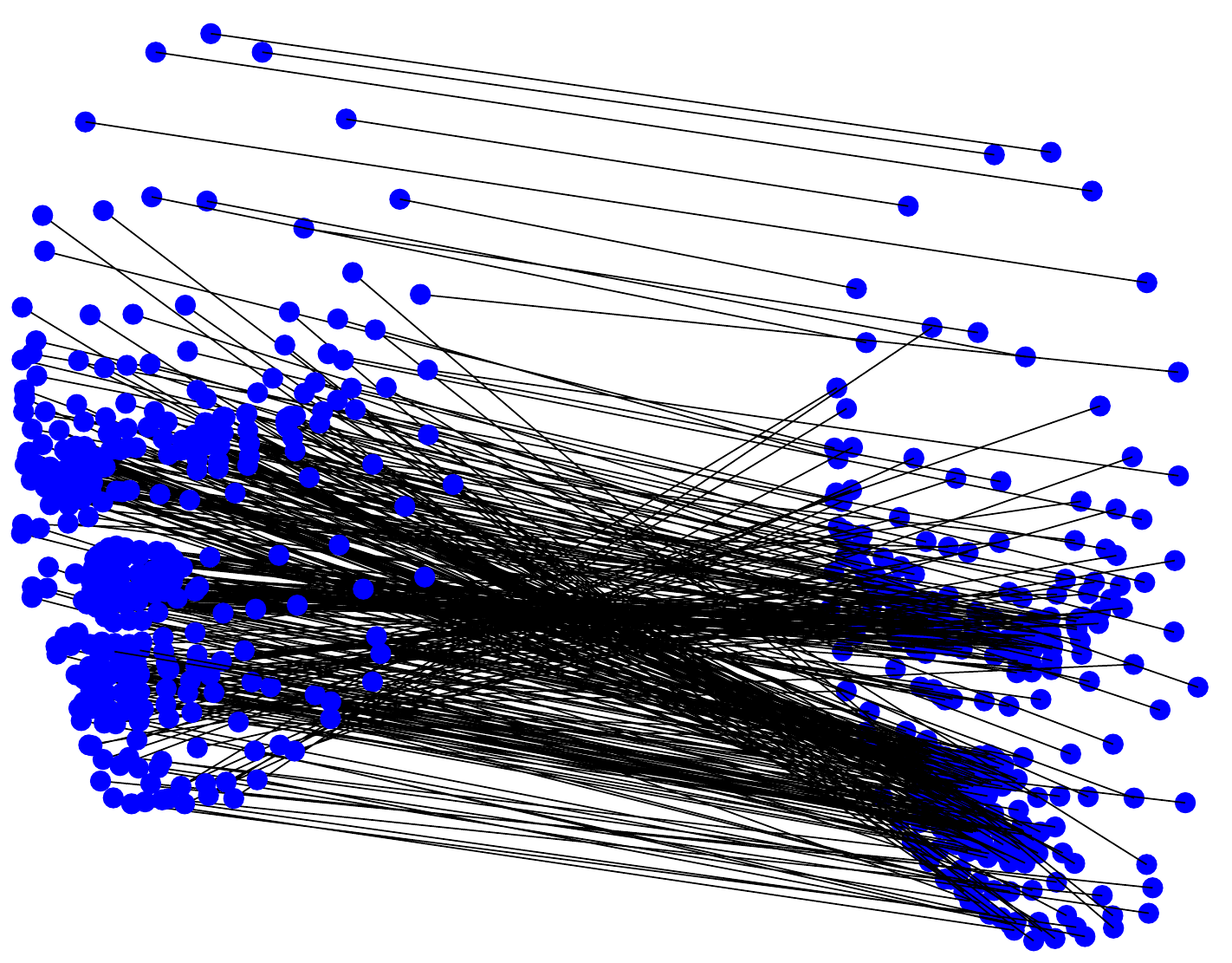}}
\subfigure[FastGA]{
\includegraphics[height=4cm,width=7cm]{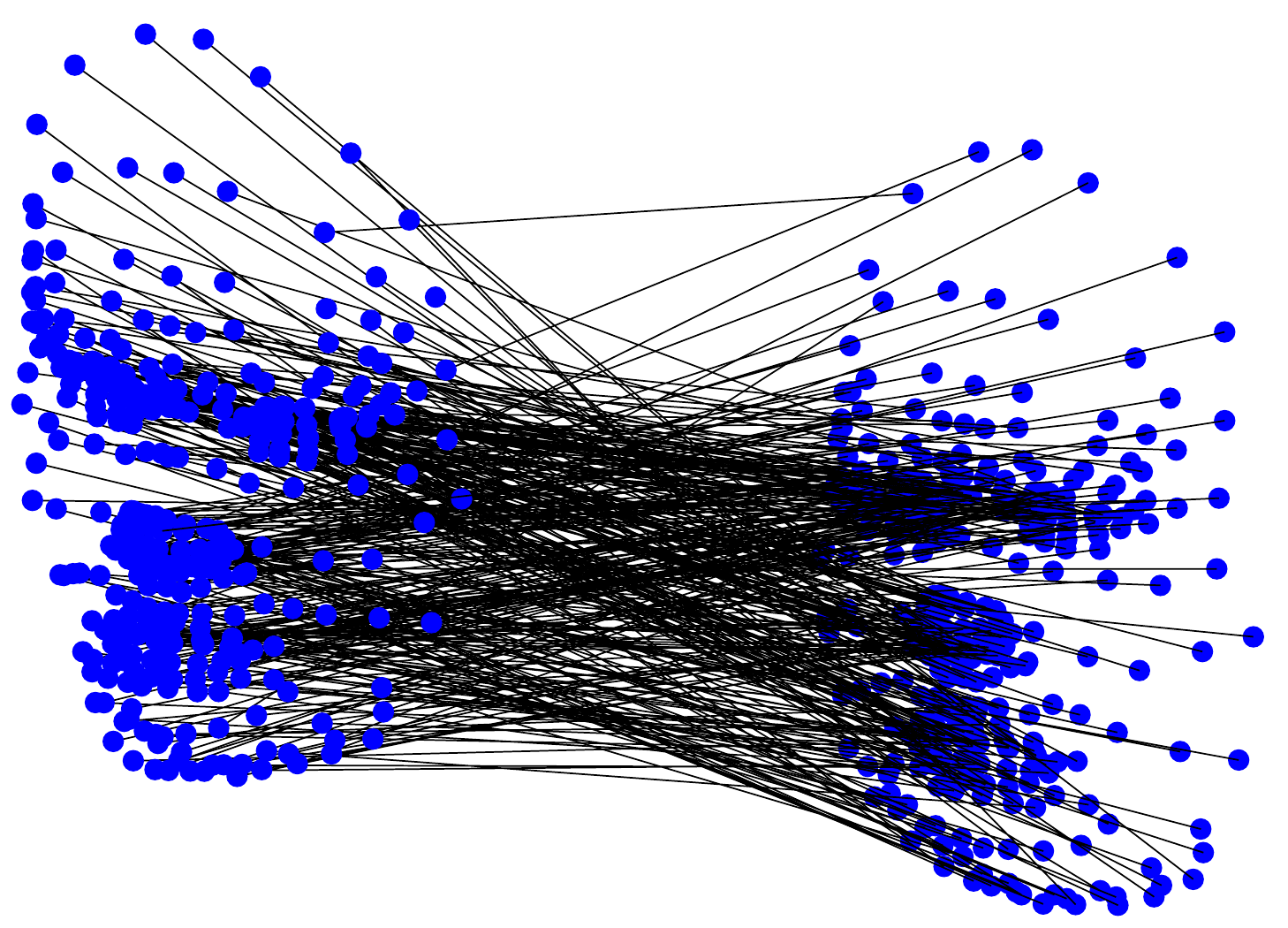}
\includegraphics[height=4cm,width=7cm]{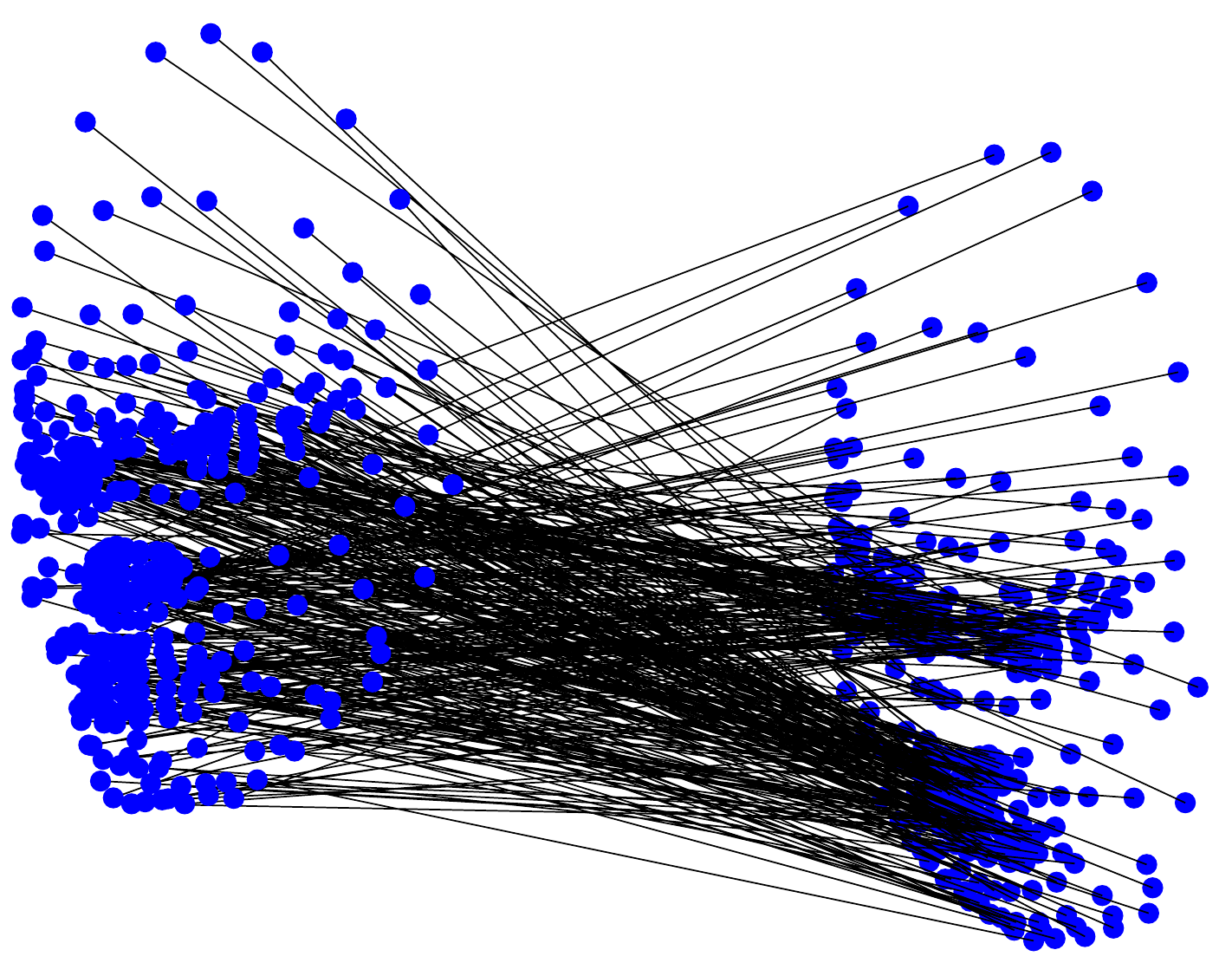}}
\subfigure[Umeyama]{
\includegraphics[height=4cm,width=7cm]{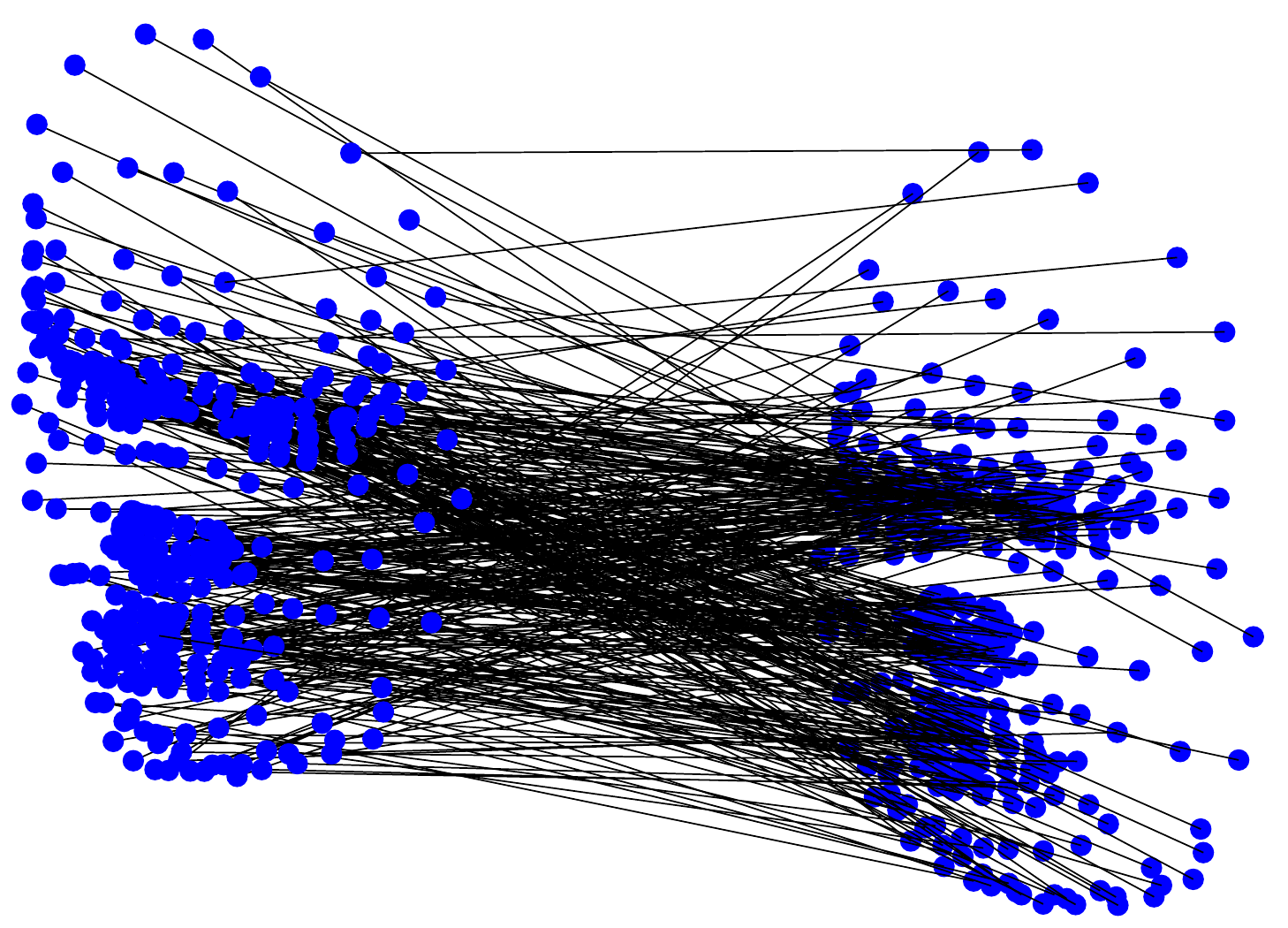}
\includegraphics[height=4cm,width=7cm]{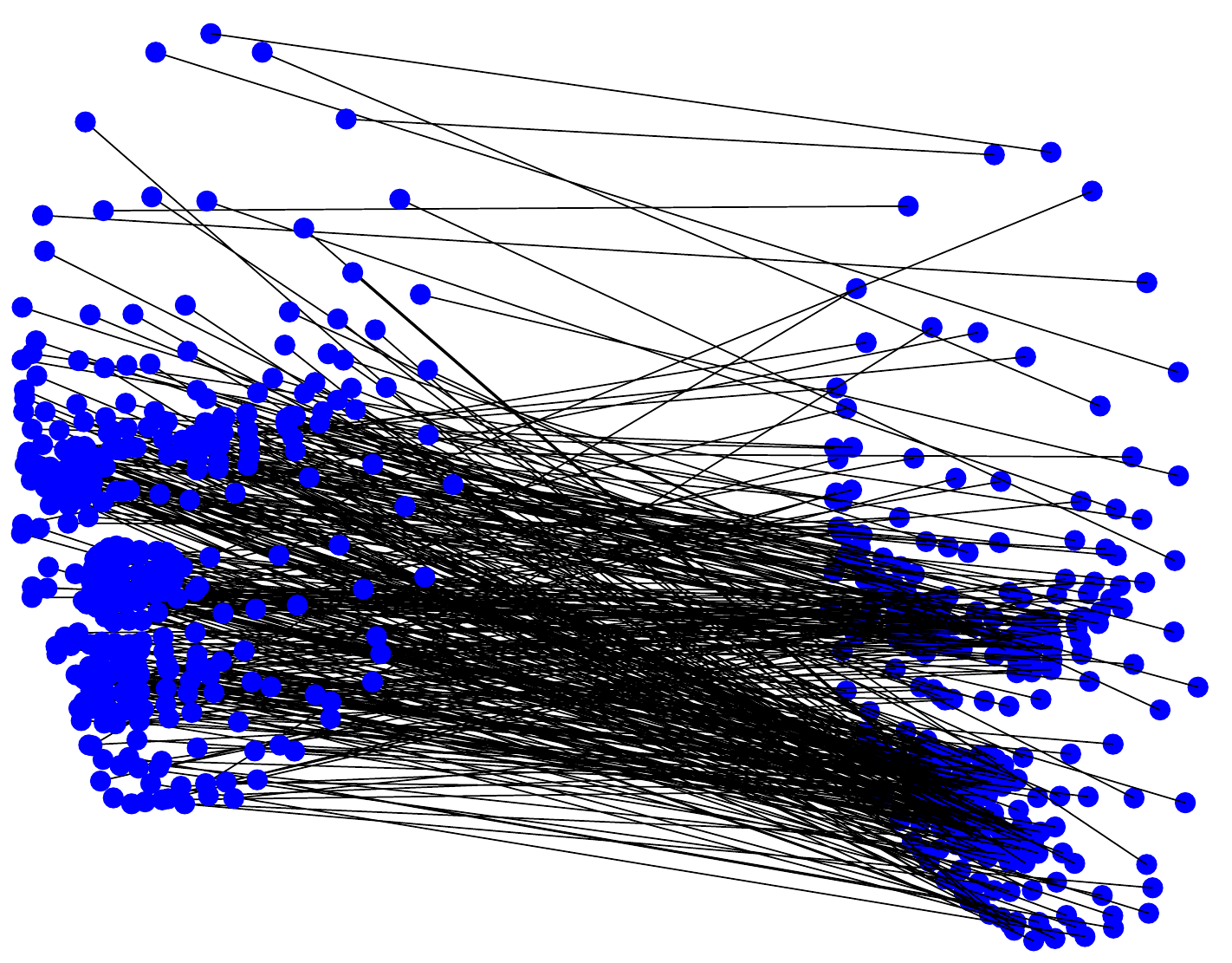}}
\subfigure[FastPFP (Ours)]{
\includegraphics[height=4cm,width=7cm]{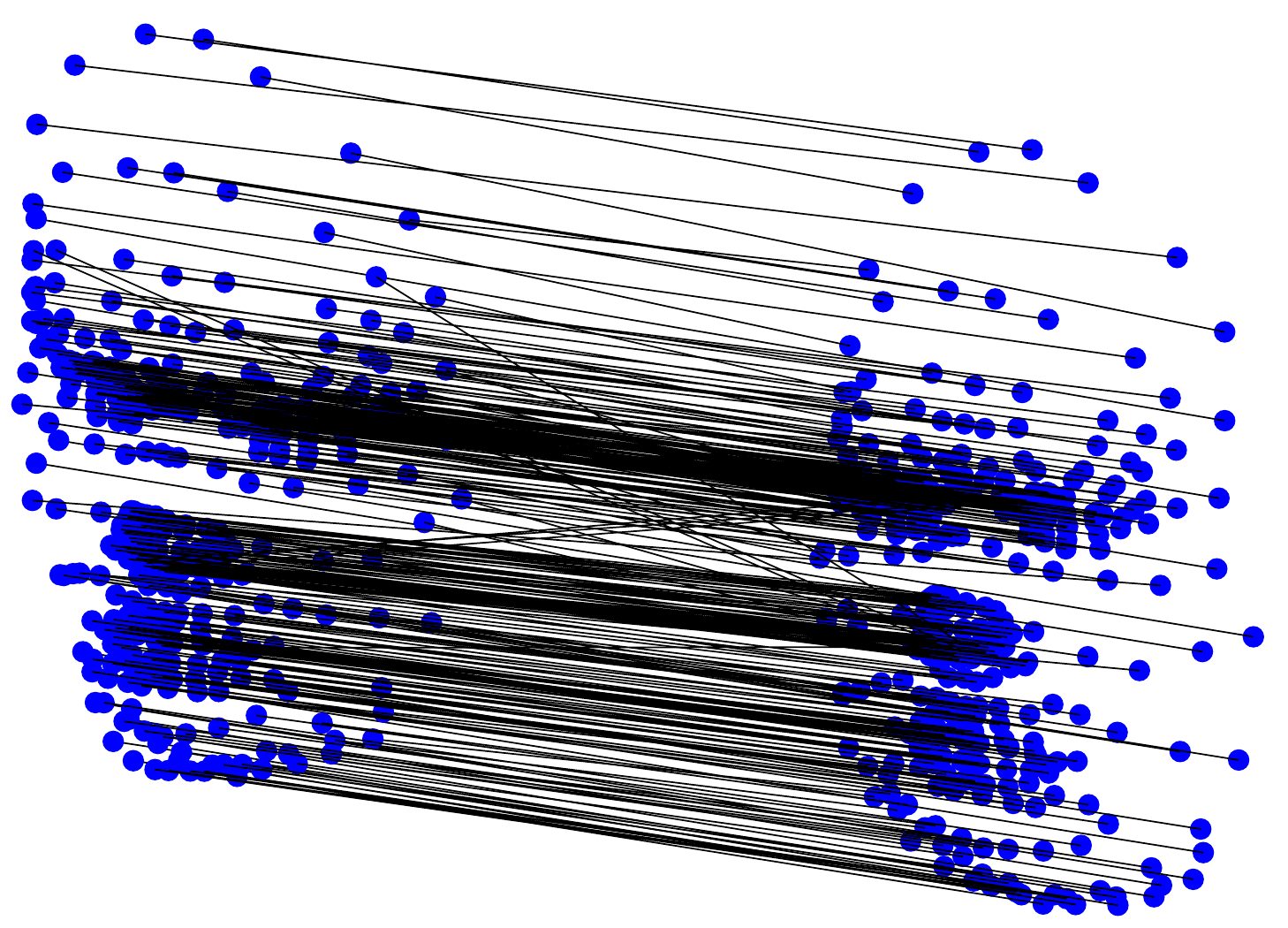}
\includegraphics[height=4cm,width=7cm]{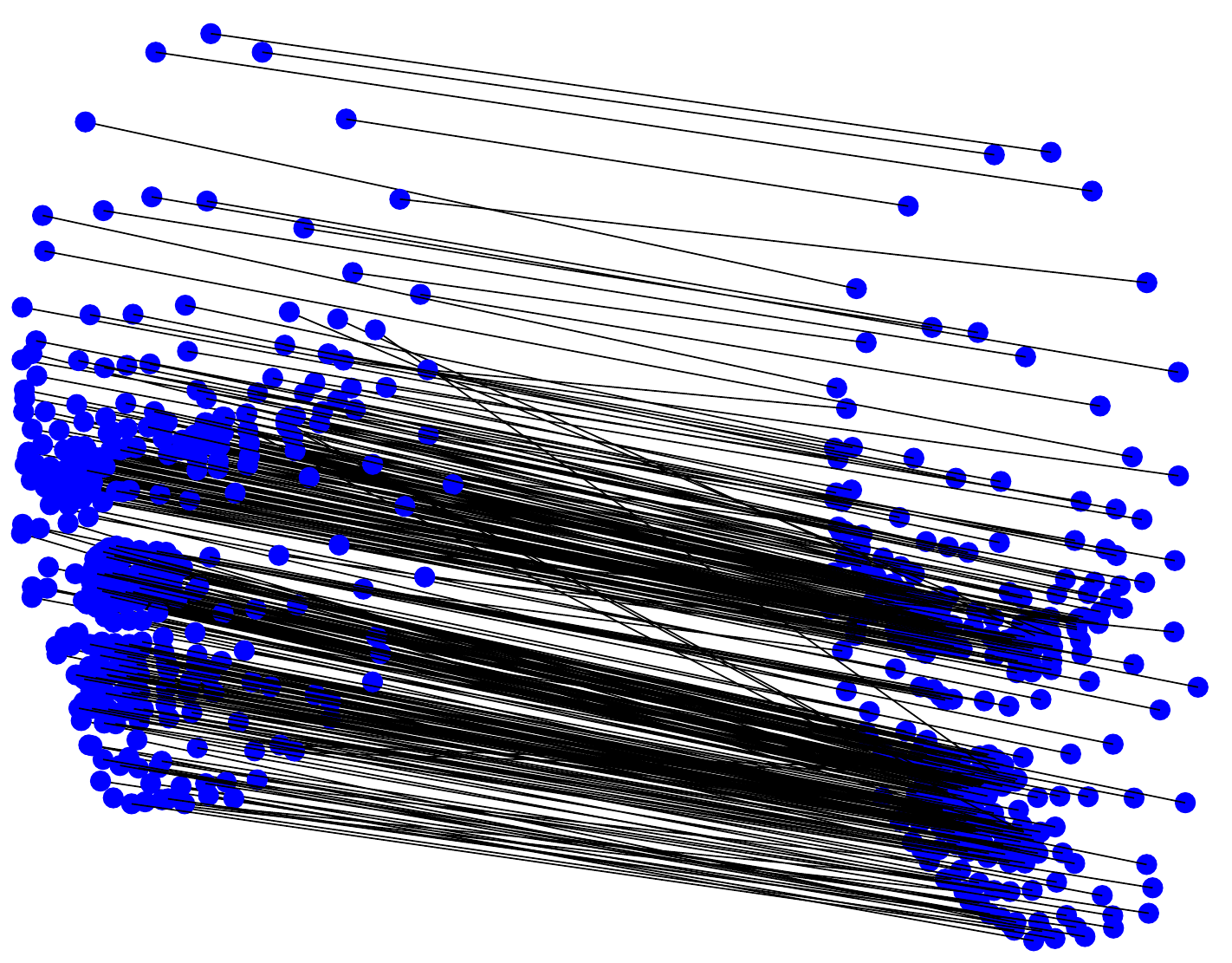}}
\caption{Graph from 3D points matching}\label{face}
\end{figure*}
\begin{figure*}
\centering
\subfigure[Stanford Bunny (1022 nodes vs. 1022 nodes)]{\includegraphics[height=6cm,width=15cm]{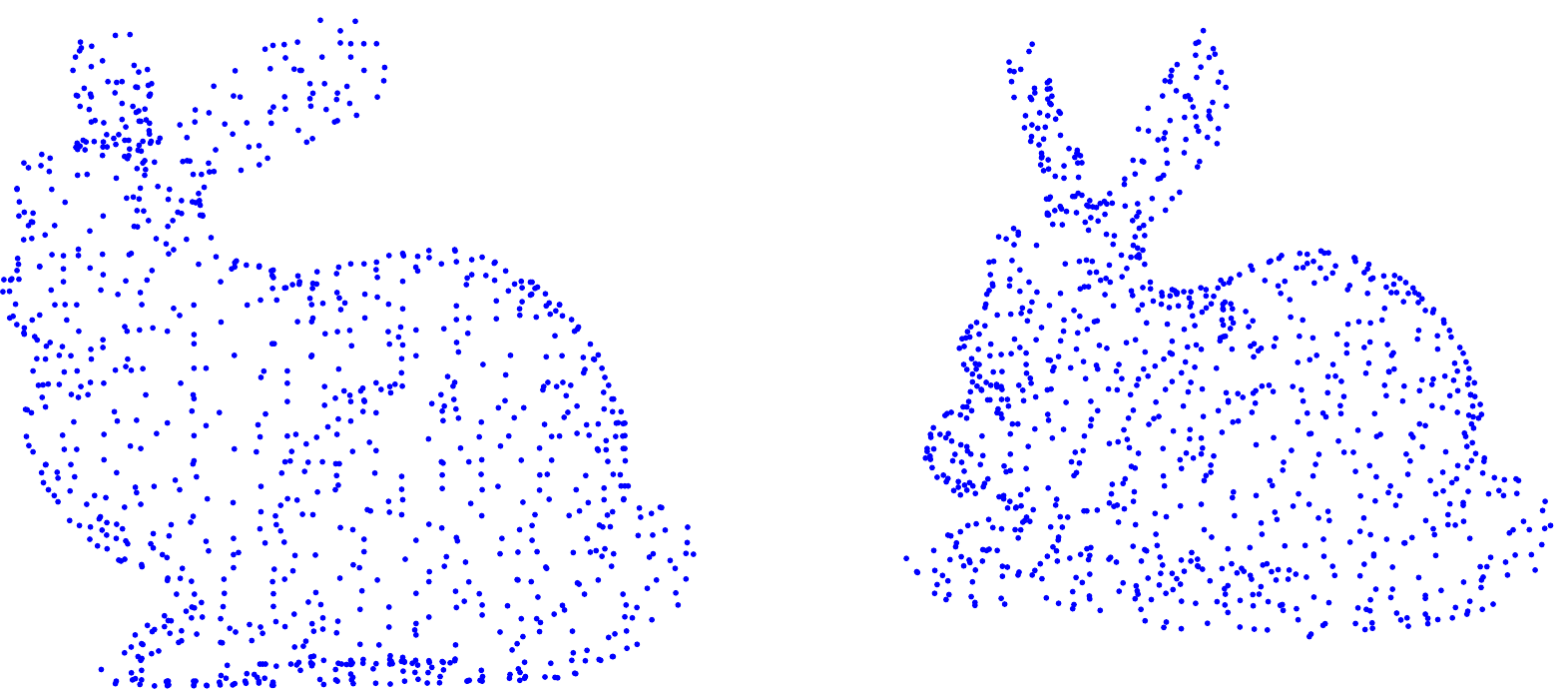}}
\subfigure[PG]{
\includegraphics[height=3cm,width=5cm]{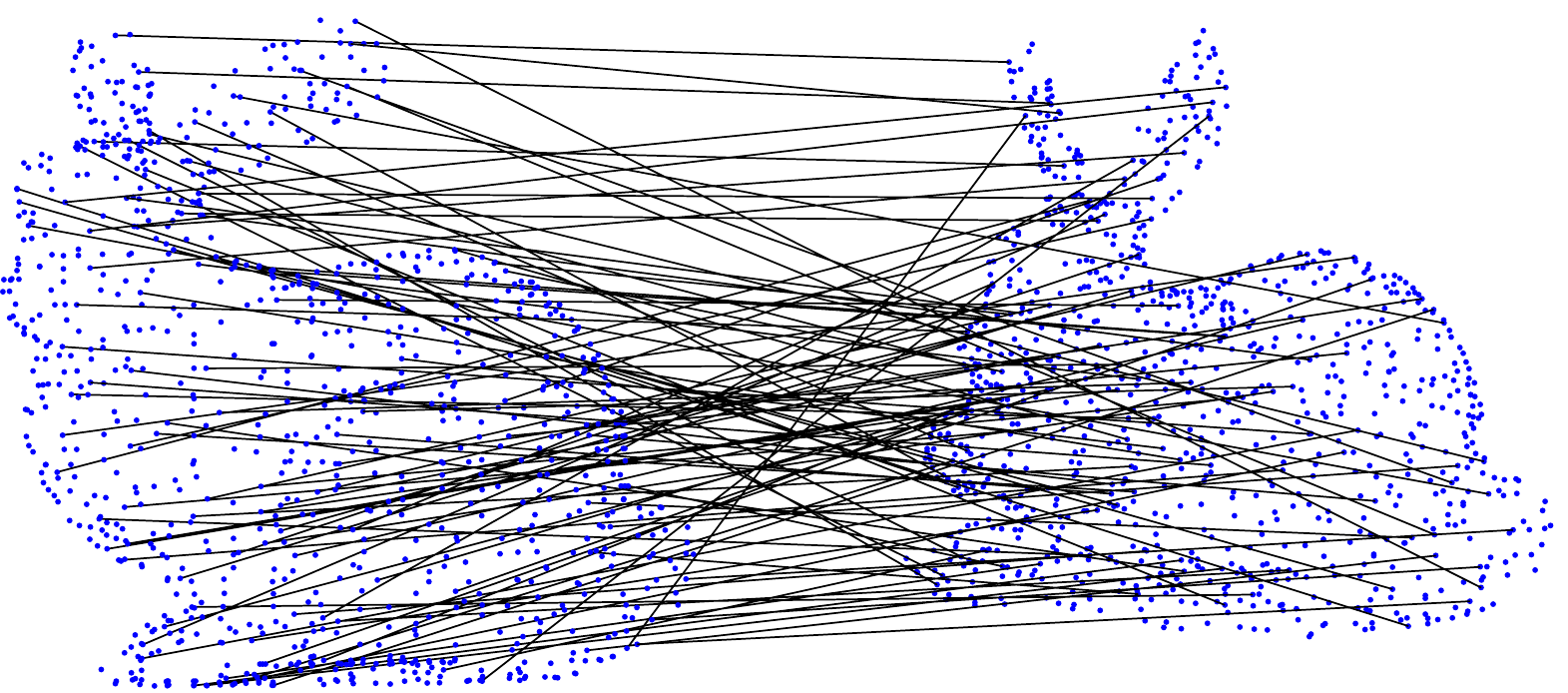}
\includegraphics[height=3cm,width=5cm]{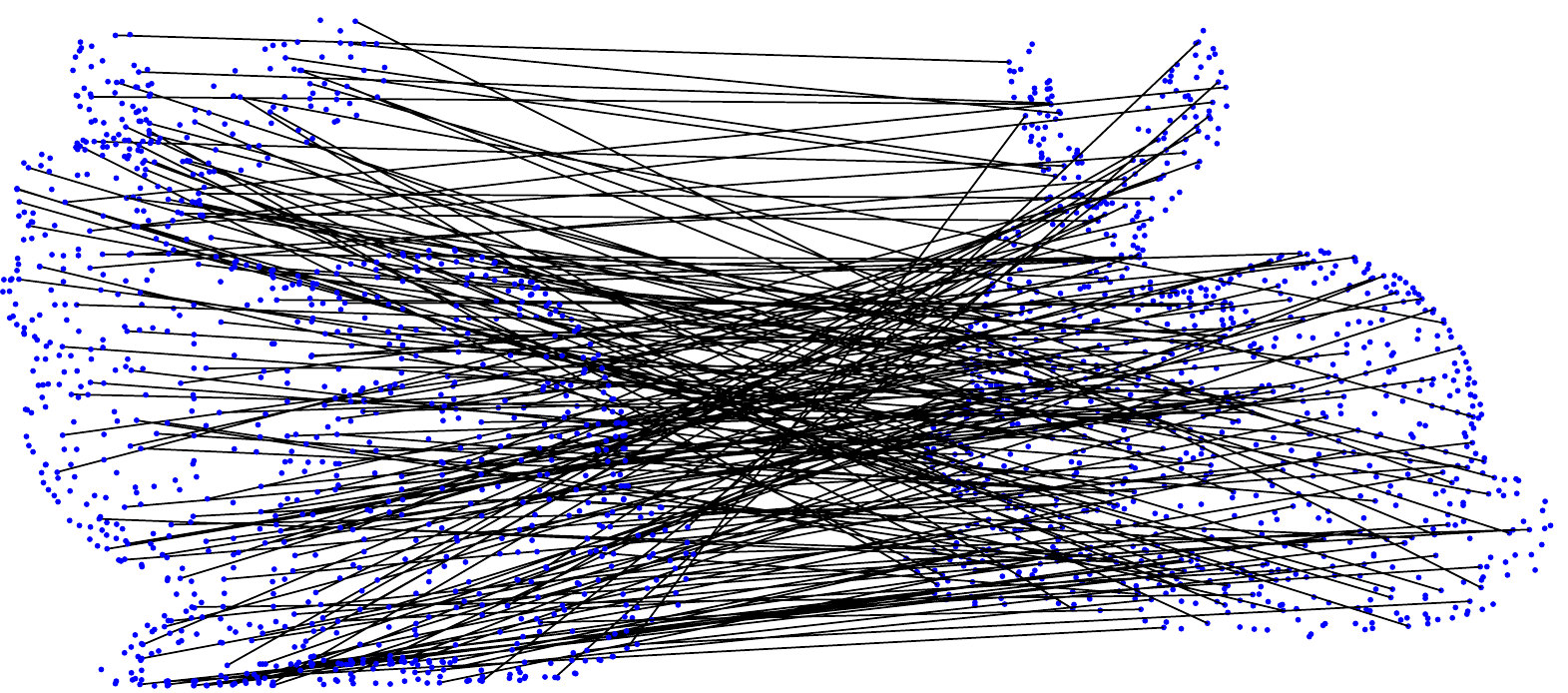}
\includegraphics[height=3cm,width=5cm]{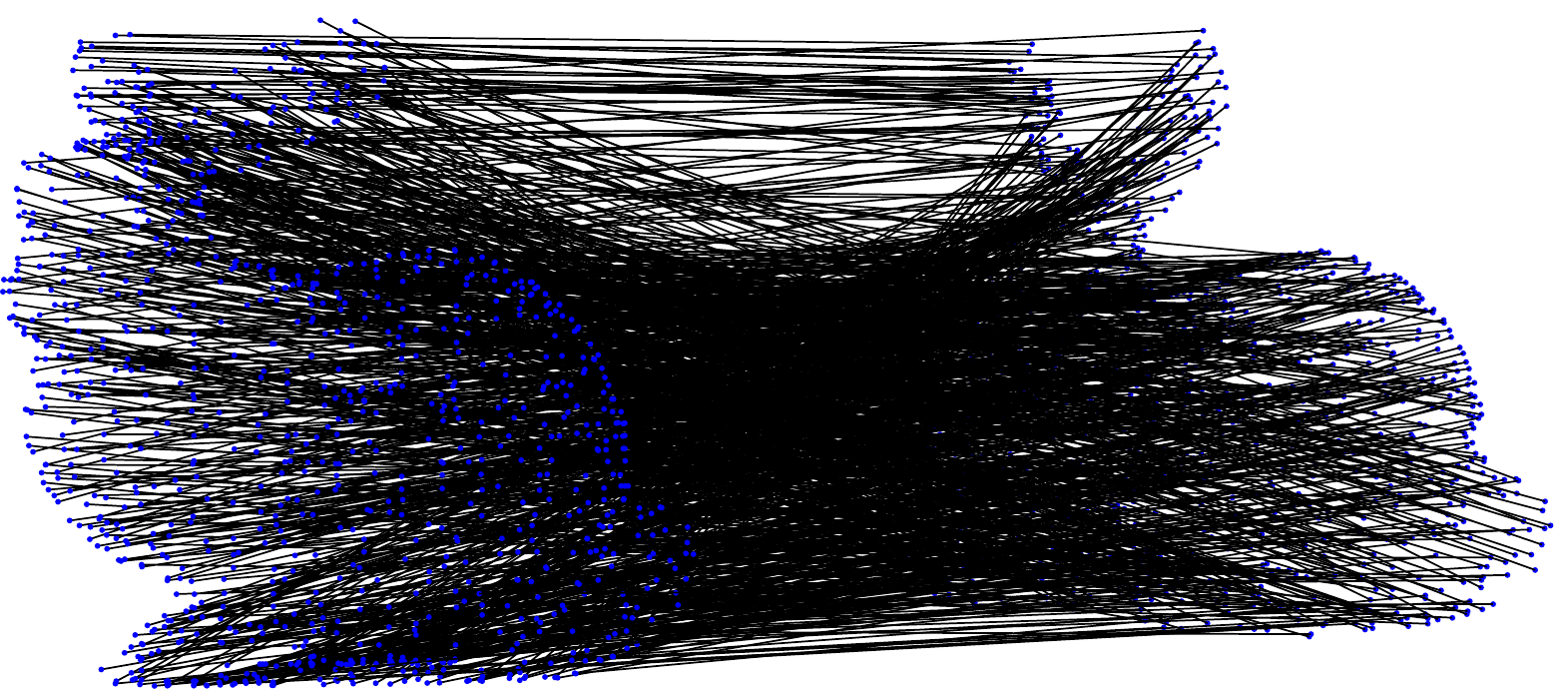}
}
\subfigure[FastGA]{
\includegraphics[height=3cm,width=5cm]{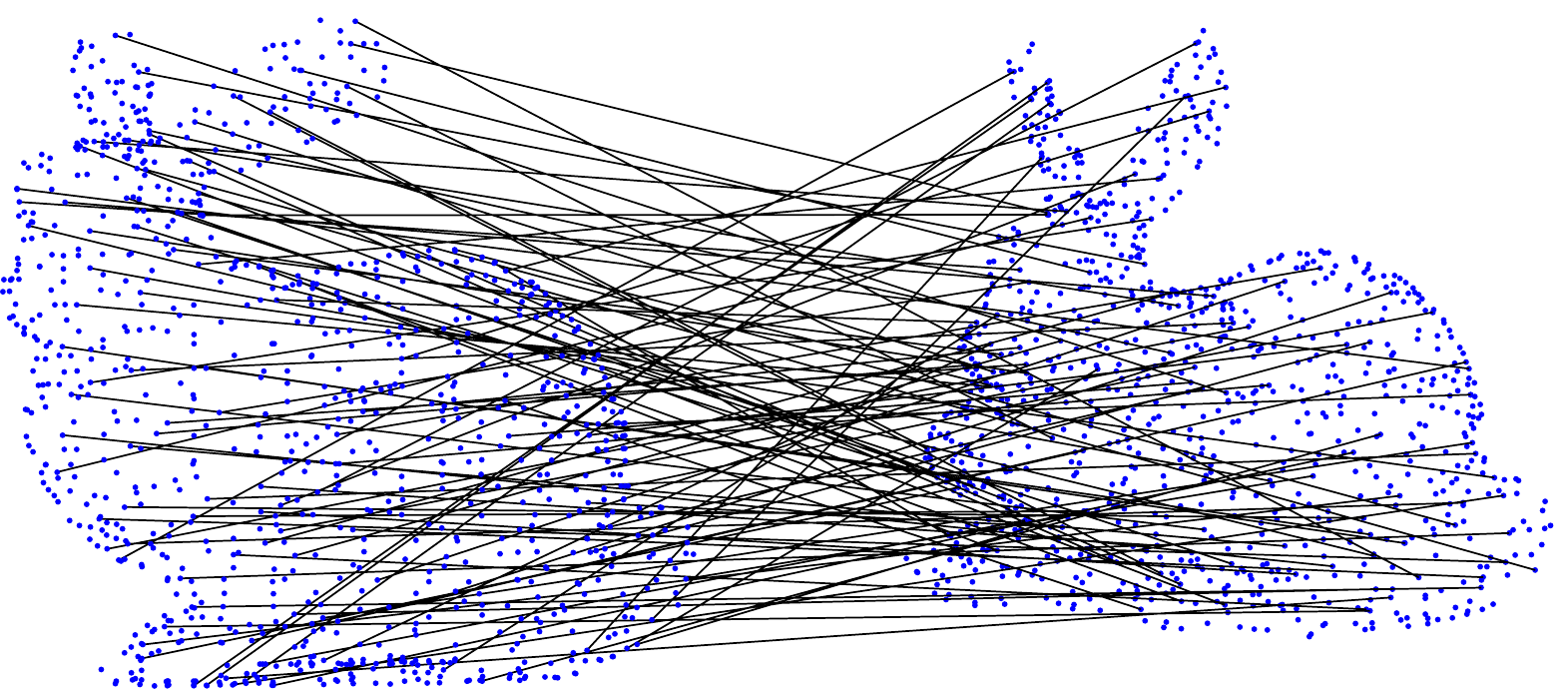}
\includegraphics[height=3cm,width=5cm]{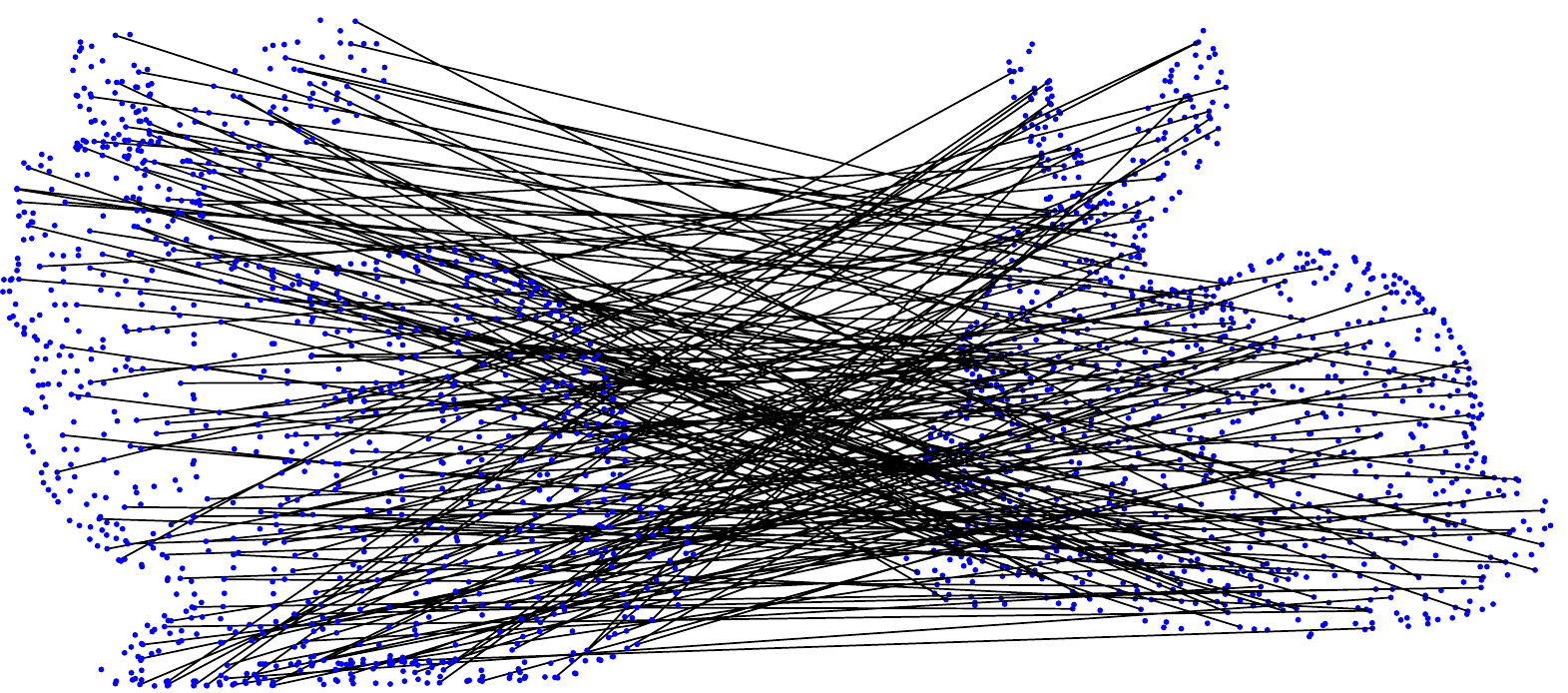}
\includegraphics[height=3cm,width=5cm]{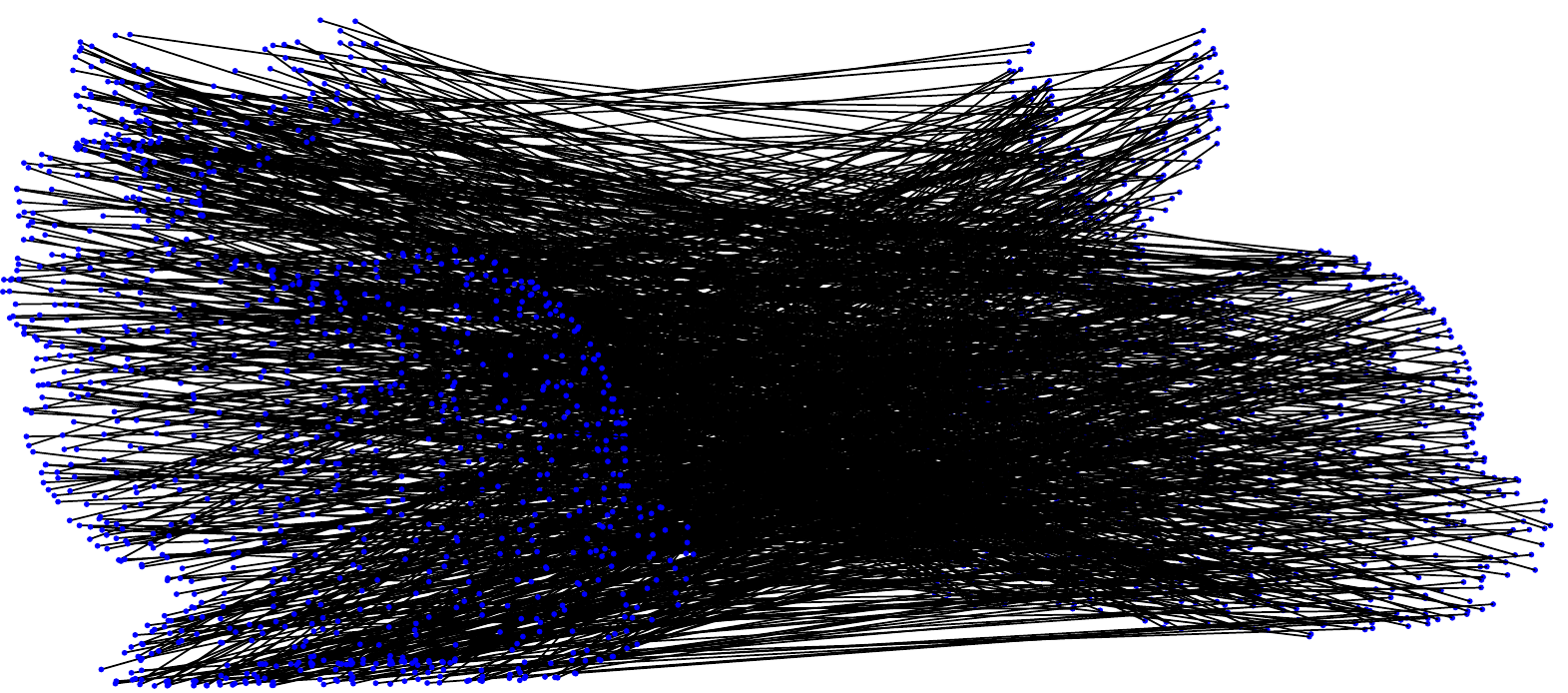}
}
\subfigure[Umeyama]{
\includegraphics[height=3cm,width=5cm]{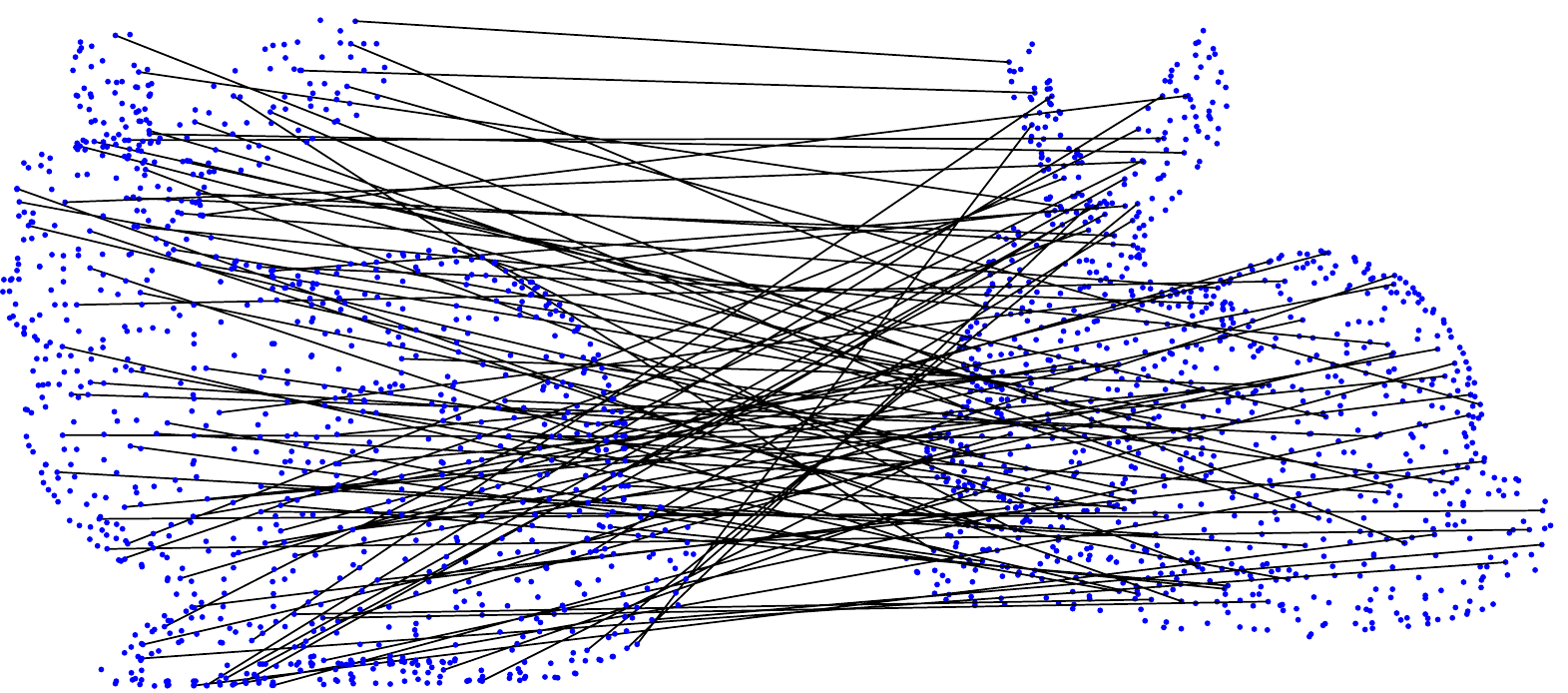}
\includegraphics[height=3cm,width=5cm]{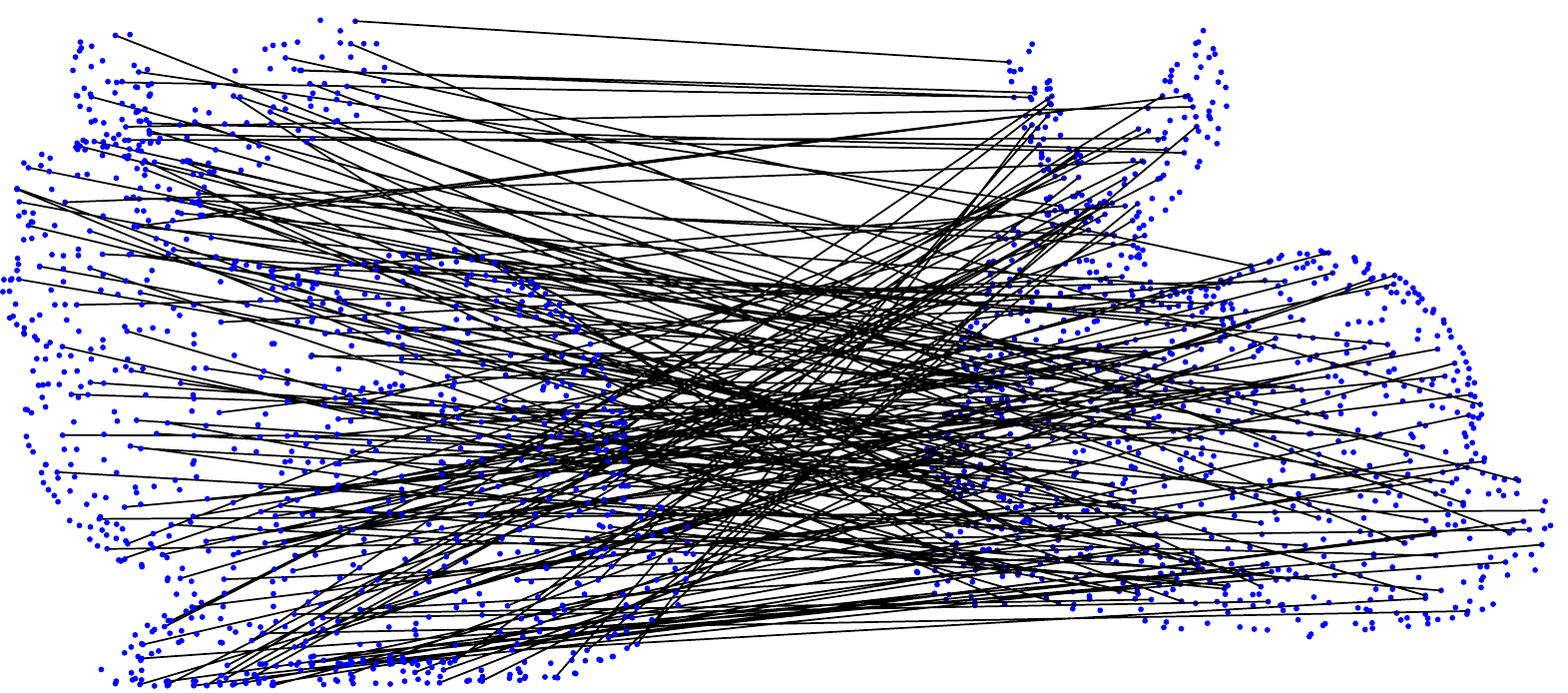}
\includegraphics[height=3cm,width=5cm]{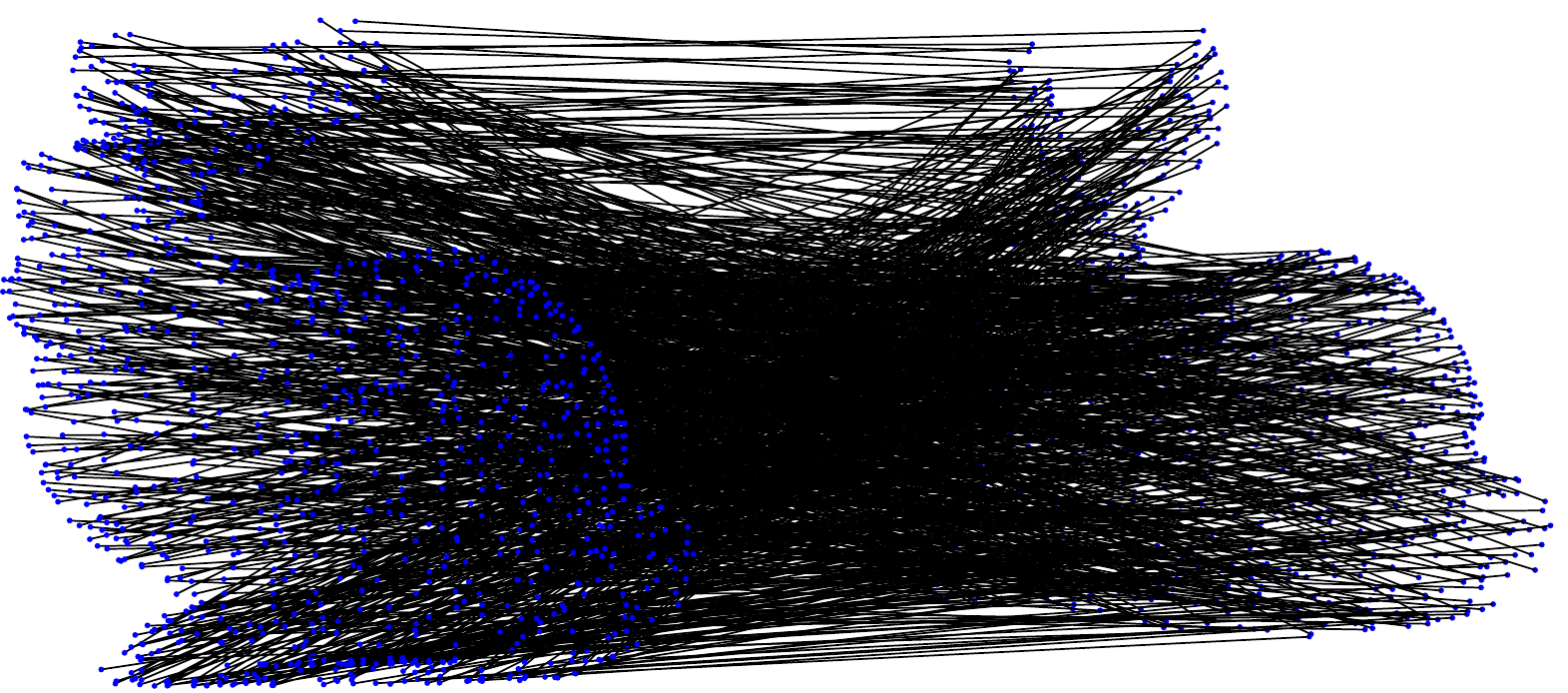}
}
\subfigure[FastPFP (Ours)]{
\includegraphics[height=3cm,width=5cm]{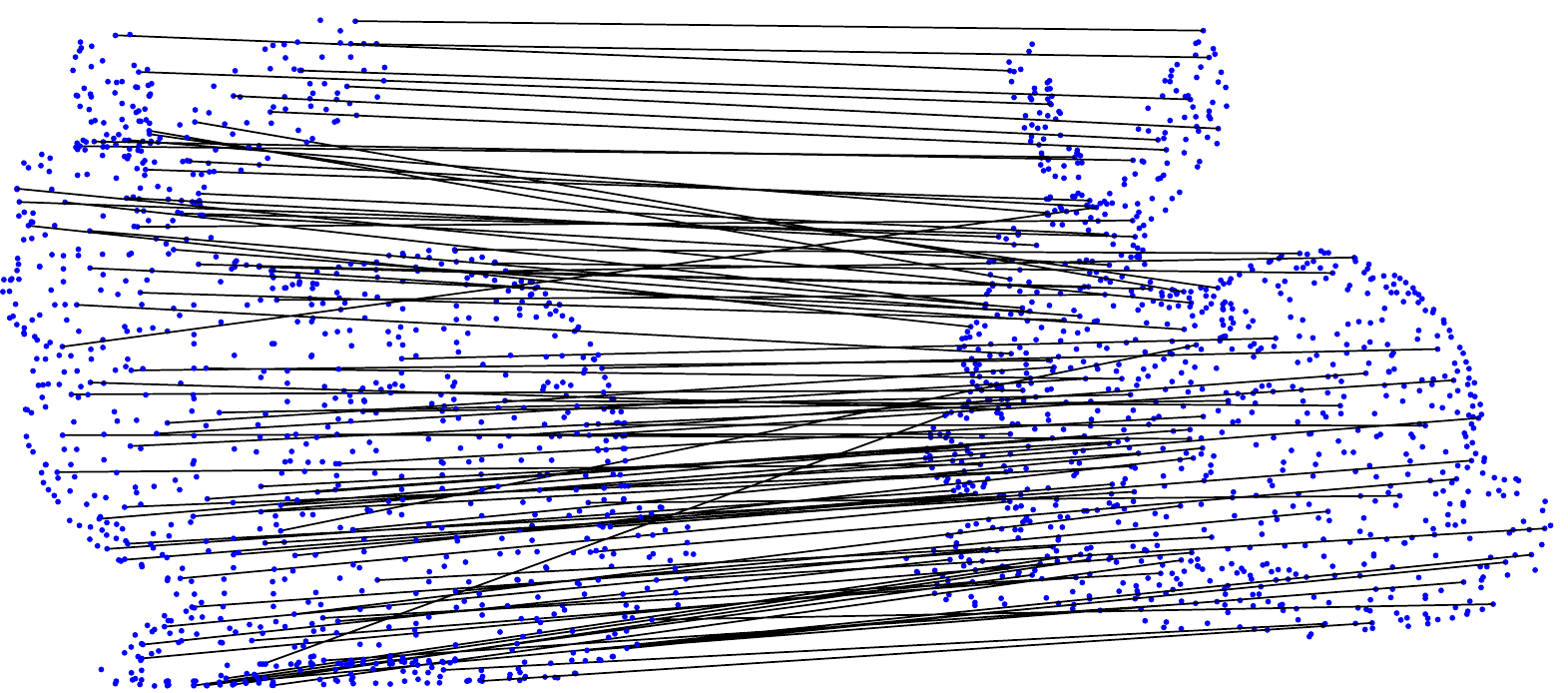}
\includegraphics[height=3cm,width=5cm]{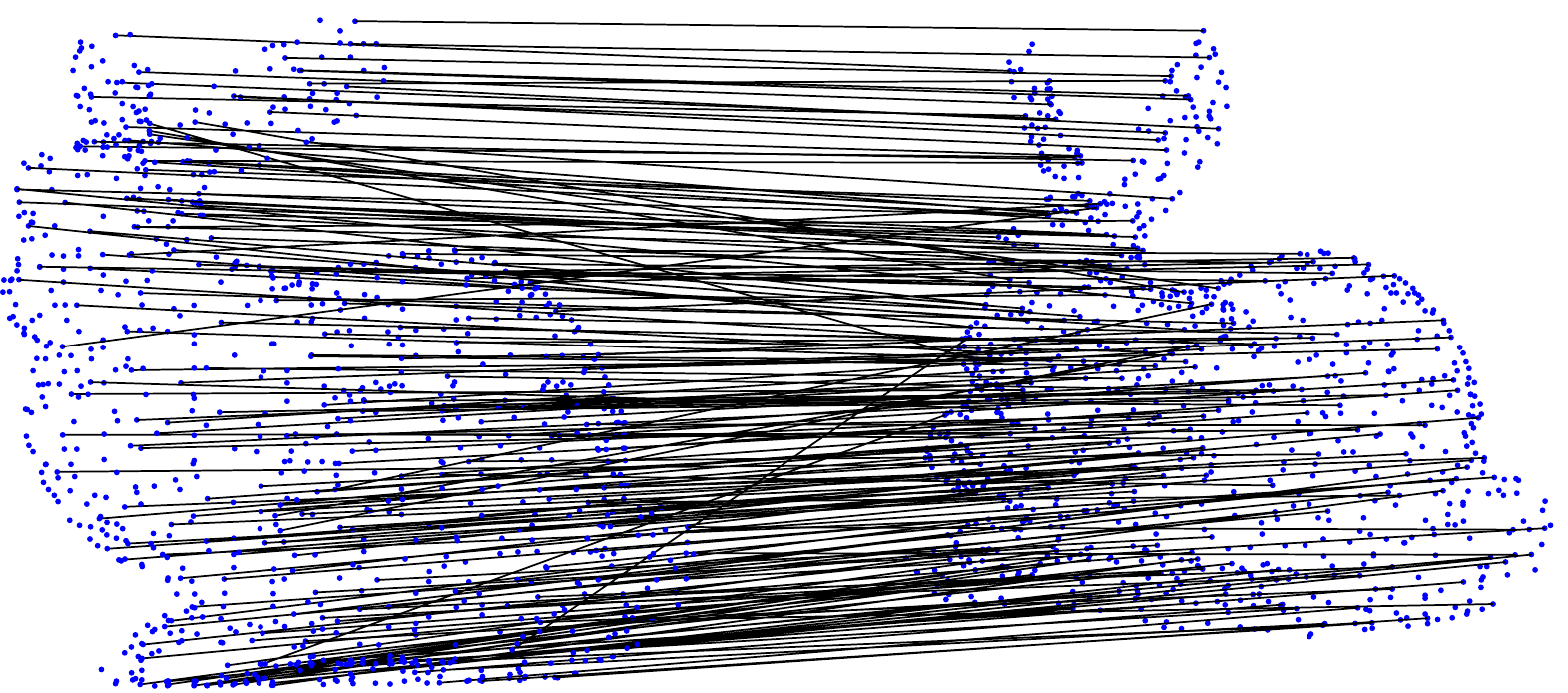}
\includegraphics[height=3cm,width=5cm]{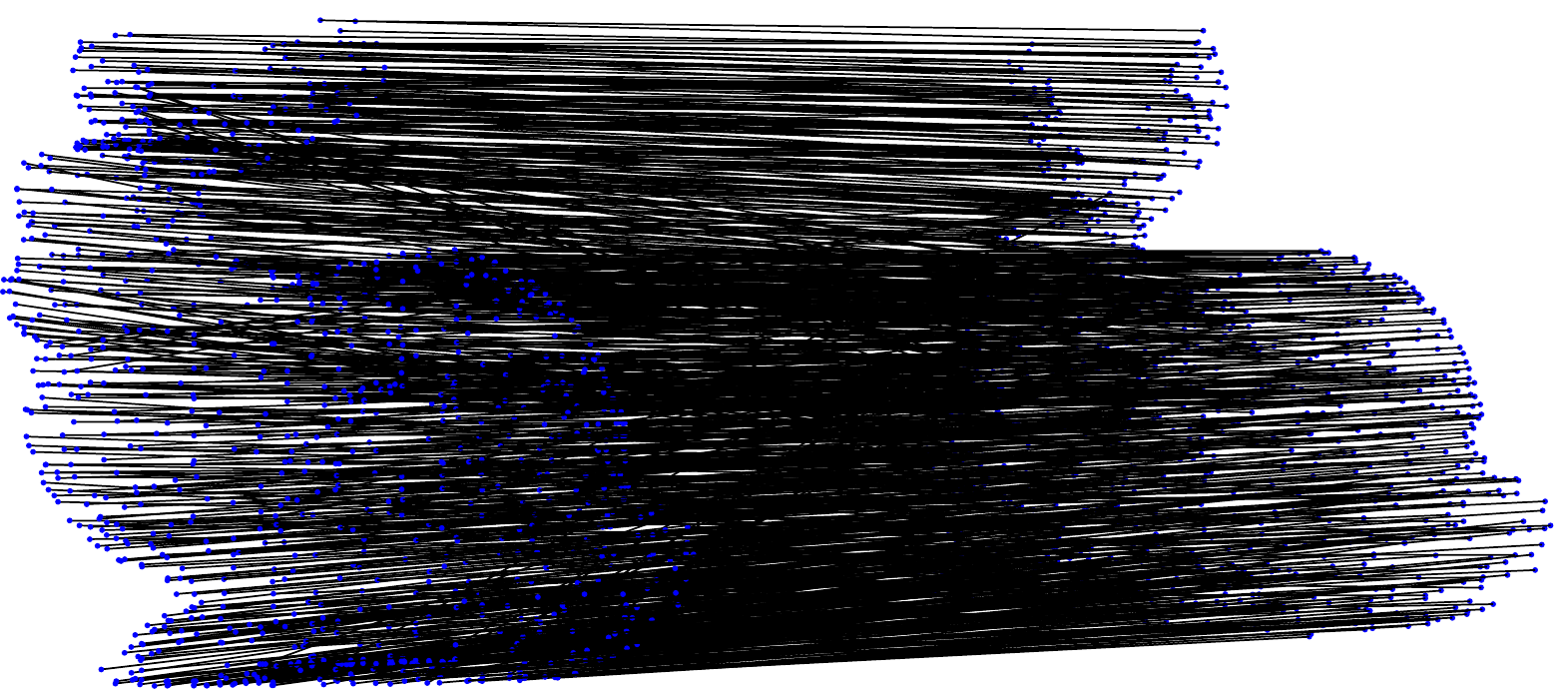}
}
\caption{Graph from 3D points matching: Stanford Bunny. For (b)-(e), left column: 10\% matching displayed. Middle column: 20\%
matching displayed. Right column: 100\% matching displayed.}\label{bunny}
\end{figure*}
In this set of experments, weighted graphs were constructed from 3D points. The Face 50 and Stanford Bunny 3D points datasets\footnote{Both datasets can be downloaded in \url{https://sites.google.com/site/myronenko/}} were used. In the Face 50 dataset, each face has 392 points, all of which were used for the construction of graphs. In the Stanford Bunny 3D data, each set of 3D points was downsampled to 1022 points. Graphs, whose edge represents Euclidean distance between two points, were constructed. The visualization of matching results of PG, FastGA, FastPFP and Umeyama's method are shown in Fig. \ref{face} and \ref{bunny}. The matching error (measured as $\| A-XA'X^T \|_F^2$) and runtime are shown in Table \ref{table_runtime} and \ref{table_error}. Again, FastPFP outperformed the others in both speed and accuracy. FastPFP is about 15 $\sim$ 20 times faster than FastGA.

\begin{table}
\begin{center}
\caption{Graphs from 3D points: runtime (sec)}\label{table_runtime}
\begin{tabular}{c|c|c|c|c}
\hline
Graph pair & PG & FastGA & Umeyama & FastPFP\\
\hline
Face pair 1 & 2.90 & 15.05 & 18.40 & $\mathbf{0.69}$\\
Face pair 2 & 3.06 & 15.53 & 18.23 & $\mathbf{0.69}$\\
Stanford Bunny & 21.74 & 118.55 & 1027.02 & $\mathbf{8.75}$\\
\hline
\end{tabular}
\caption{Graphs from 3D points: matching error ($10^5$)}\label{table_error}
\begin{tabular}{c|c|c|c|c}
\hline
Graph pair & PG & FastGA & Umeyama & FastPFP \\
\hline
Face pair 1 & 6.30 & 14.3 & 12.5 & $\mathbf{0.78}$\\
Face pair 2 & 6.23 & 16.2 & 10.6 & $\mathbf{1.64}$\\
Stanford Bunny & 29.1 & 37.4 & 32.7 & $\mathbf{9.55}$ \\
\hline
\end{tabular}
\end{center}
\end{table}

\section{Discussion}
\subsection{Comparison To FastGA}
Although FastGA and FastPFP both have time complexity $O(n^3)$ per
iteration, FastPFP is empirically about $3\sim20$ times faster than FastGA. This
is mainly due to the slow convergence of FastGA. Despite the theoretical work
on convergence of FastGA~\cite{GAcon}, empirically FastGA does not have good
convergence speed, as observed in the experiments. To see that, we show the
magnitude of parameter $\epsilon_1=\max(|X^{(t+1)}-X^{(t)}|)$ for
image sequence (image 0 vs. image 110) and real image (Graffiti) for
20 iterations in Fig. \ref{convergence}. Besides good convergence,
FastPFP only requires simple arithmetic operations (addition,
multiplication and maximum element search) while FastGA requires
additionally exponentiation and logarithm (for numerically stable
implementation).
\begin{figure}
\centering \subfigure[image sequence
sequence]{\includegraphics[height=3.5cm,width=4cm]{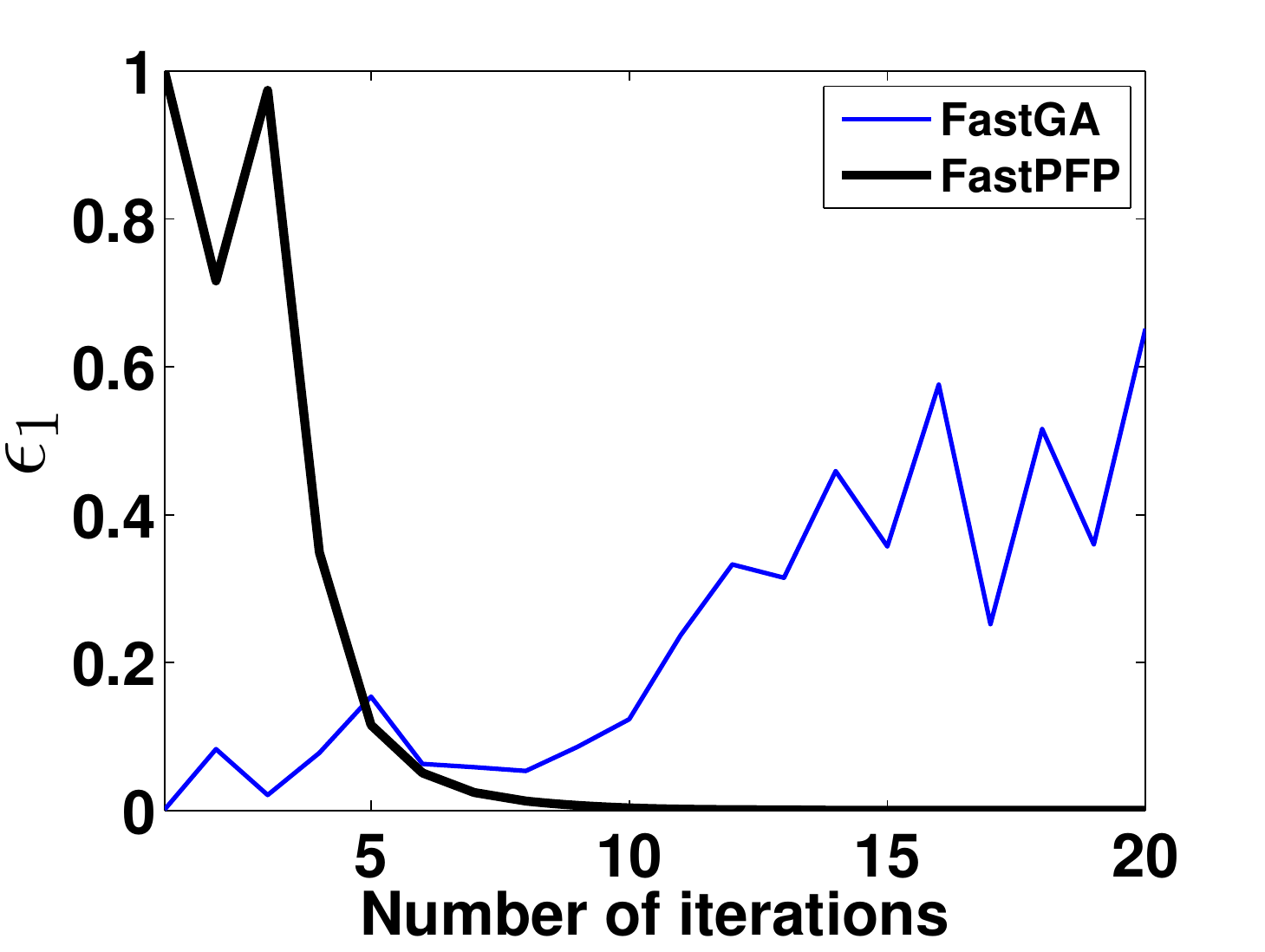}}
\subfigure[Real
image]{\includegraphics[height=3.5cm,width=4cm]{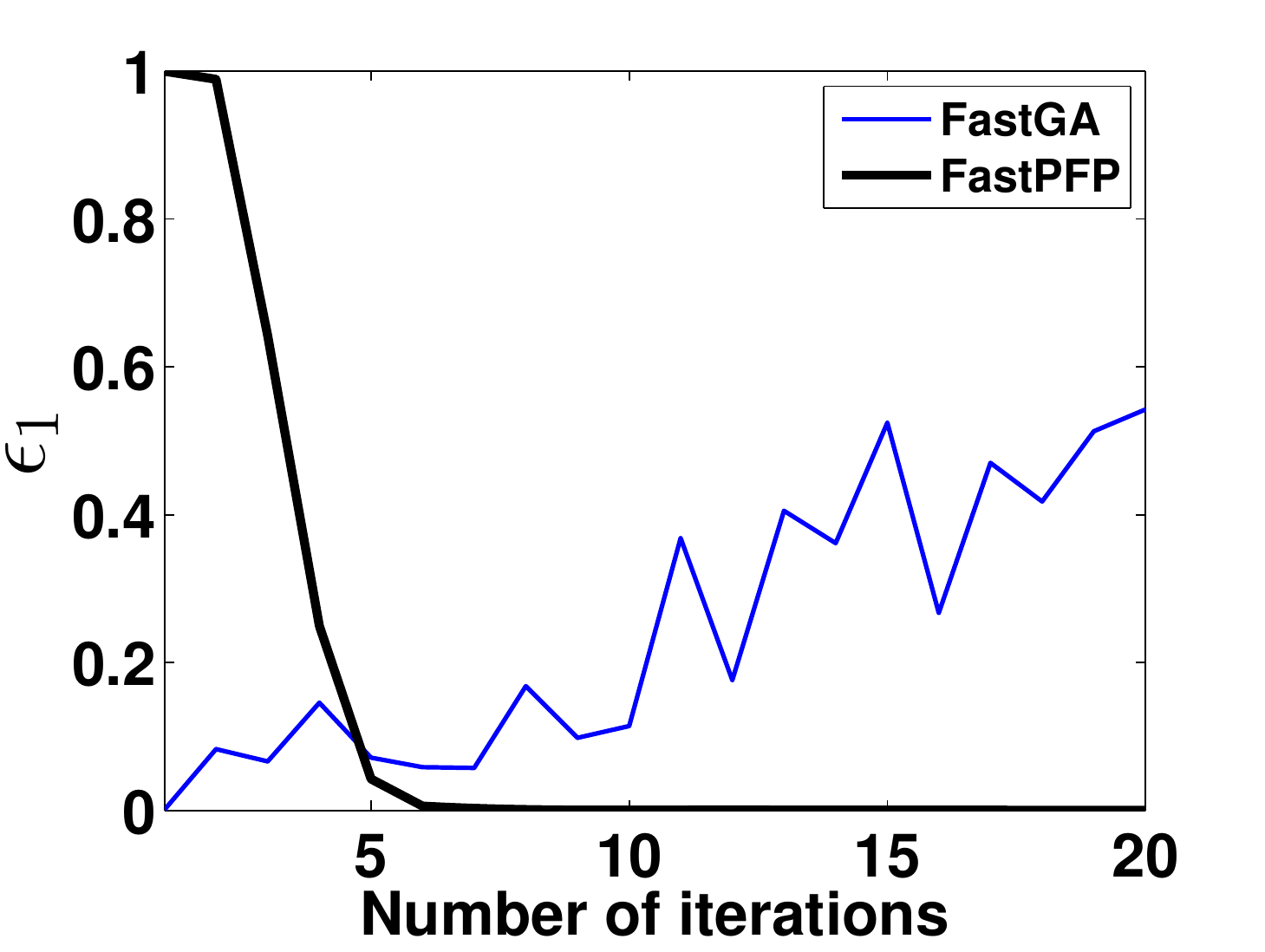}}
\caption{Convergence of FastPFP and FastGA. The vertical axis represents the value of $\epsilon_1=\max(|X^{(t+1)}-X^{(t)}|)$.}\label{convergence}
\end{figure}

\subsection{Parameter Sensitivity}\label{sec:para}
\begin{figure}
\centering
\subfigure[Image sequence]{\includegraphics[height=3.5cm,width=9cm]{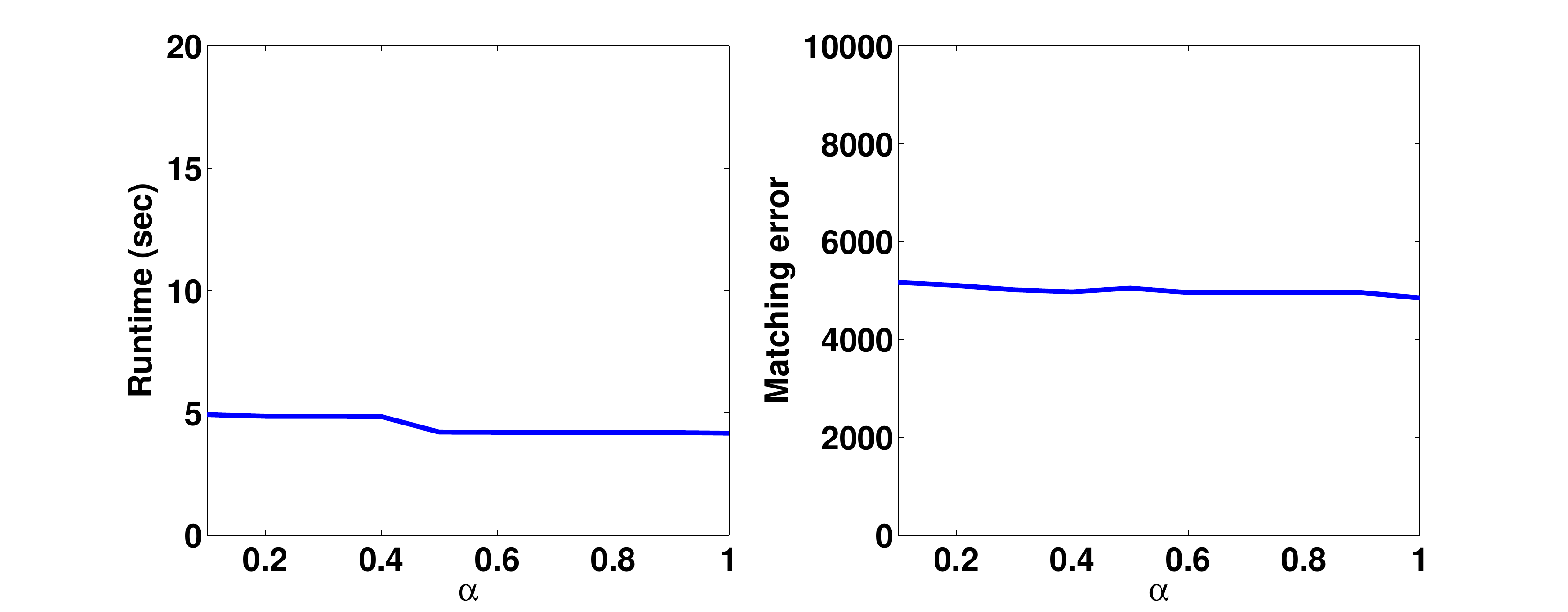}}
\subfigure[Real image]{\includegraphics[height=3.5cm,width=9cm]{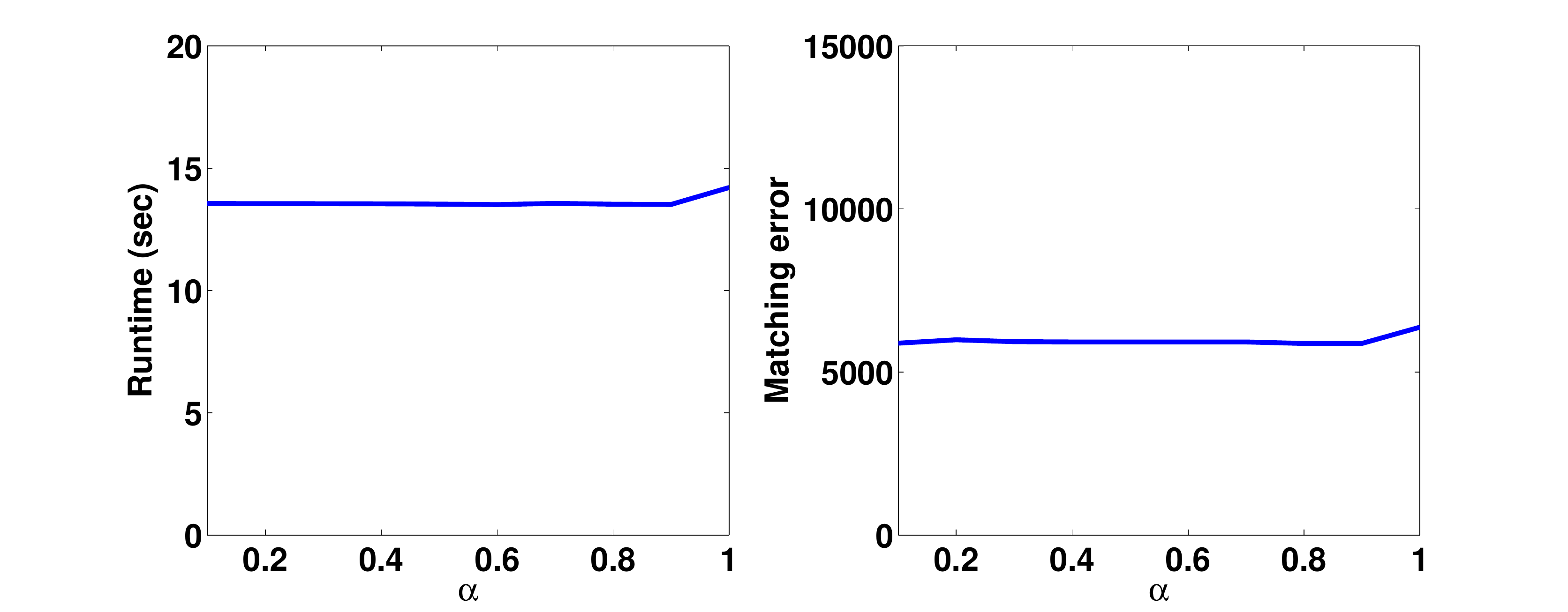}}
\caption{Performance for different values of $\alpha$}\label{parameter1}
\end{figure}

\begin{figure}
\subfigure[$\lambda=\infty$]{\includegraphics[height=3.5cm,width=9cm]{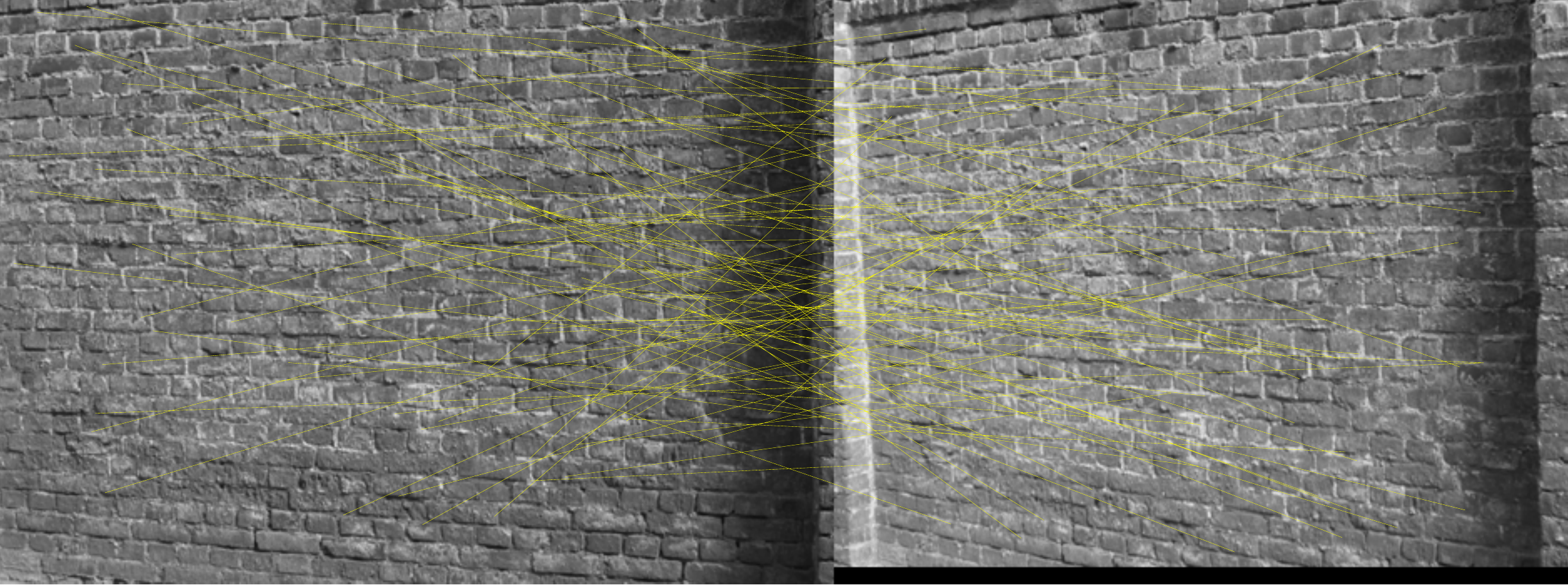}}
\subfigure[$\lambda=0$]{\includegraphics[height=3.5cm,width=9cm]{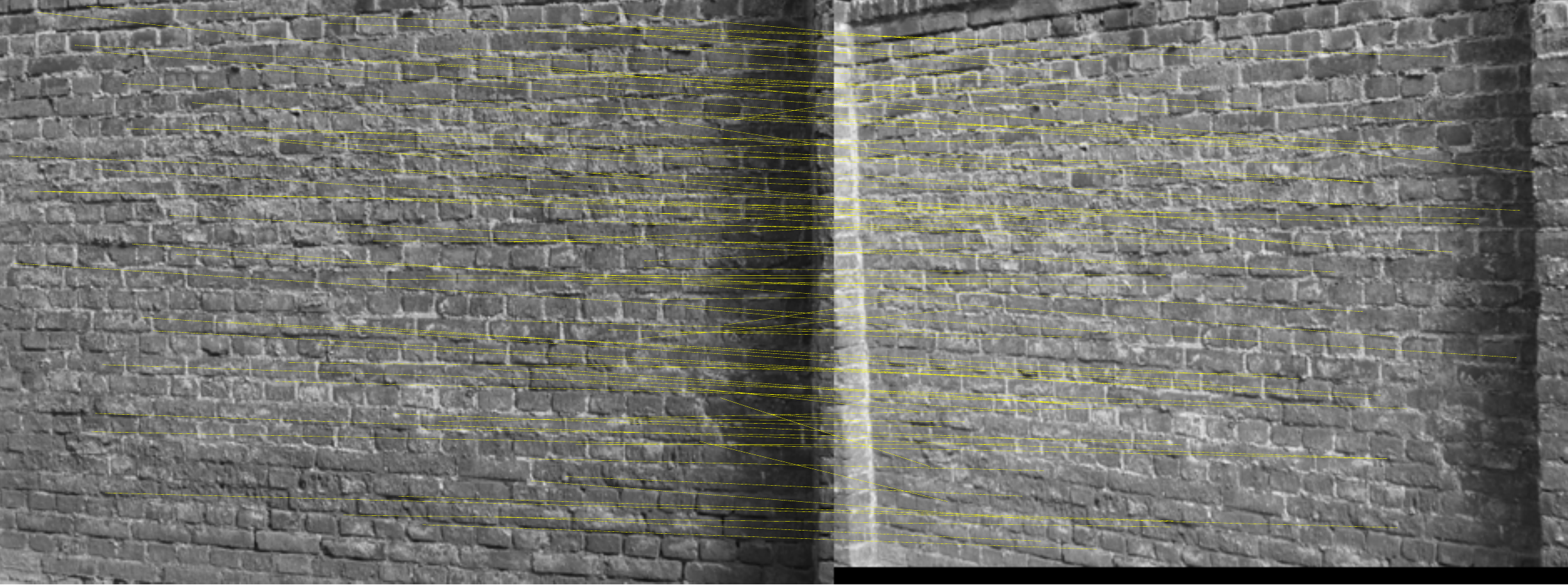}}
\subfigure[$\lambda=5$]{\includegraphics[height=3.5cm,width=9cm]{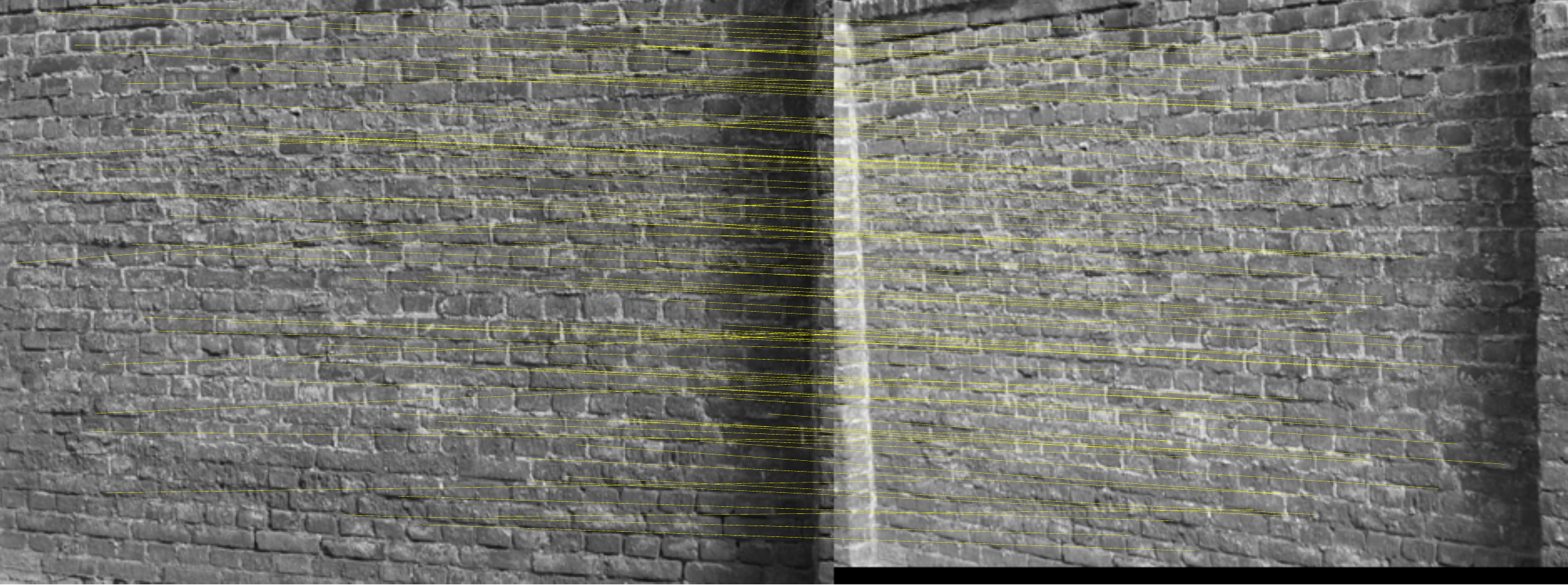}}
\subfigure[$\lambda=10$]{\includegraphics[height=3.5cm,width=9cm]{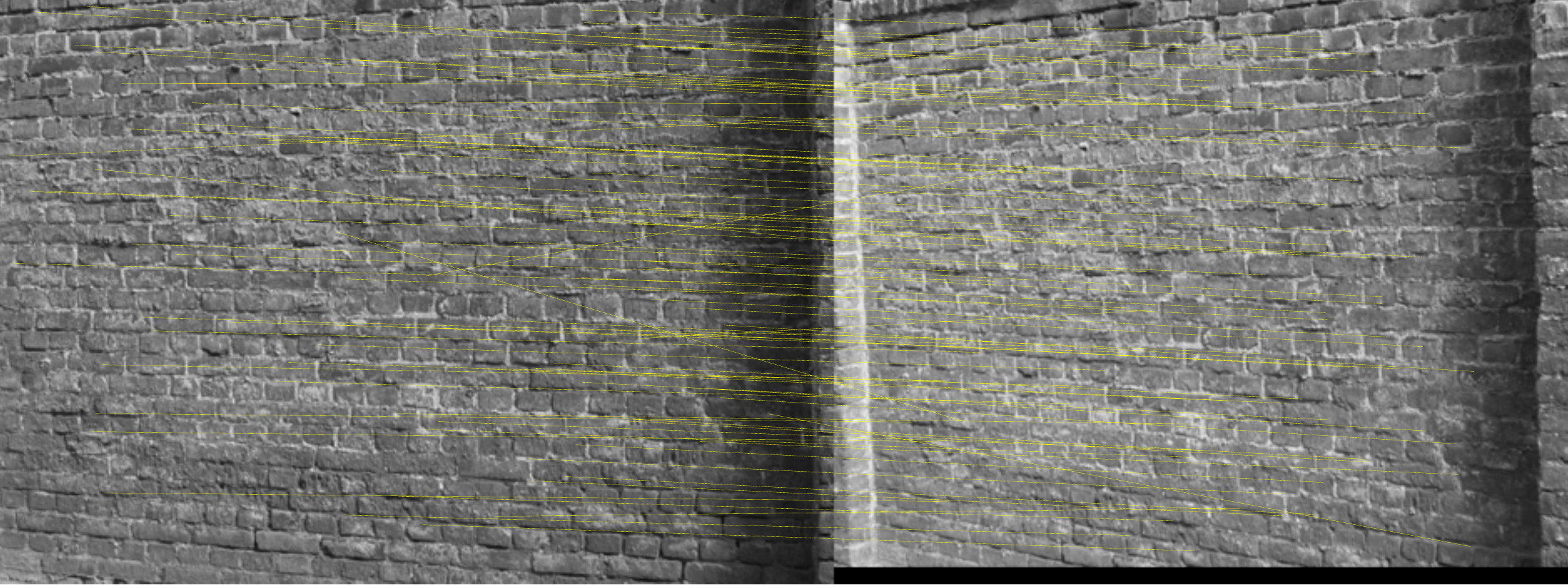}}
\caption{Performance for different values of $\lambda$. 10\% of the matching are displayed.}\label{parameter2}
\end{figure}
The performance of FastPFP is not very sensitive to the step size parameter $\alpha$, for $0 < \alpha < 1$. In Fig. \ref{parameter1}, the runtime and matching error of FastPFP in the image sequence (image 0 vs. image 110) and the real image experiment (Graffiti) are shown for different values of $\alpha$, ranging from 0.1 to 1. The performance of FastPFP is unstable for $\alpha=1$.

For attributed graph matching, the matching quality is not very sensitive to the parameter $\lambda$. We show the matching result for $\lambda=\infty$ (which means the matching objective function becomes $\| B -XB' \|^2_F$ and is solved by Hungarian method), $\lambda=0$ and $\lambda=10$. The results are similar except for $\lambda=\infty$. The matching results are shown in Fig. \ref{parameter2}.

\subsection{Limitations}
Despite the encouraging experimental results, our algorithm is derived heuristically, as many other approximate graph matching algorithms such as GA (FastGA), Umeyama's method, SM and RWRW. For these algorithms including ours, the optimality of the solutions is not theoretically guaranteed. Also, our algorithm does not use an edge compatibility matrix, due to its high computational cost. Consequently, another limitation of our algorithm is its inability to handle arbitrary edge relationship (some graphs have vector edge attributes), unlike those algorithms based an edge compatibility matrix.

\section{Conclusion}\label{sec:conc}
We proposed a new fast graph matching algorithm based on a new
projected fixed-point method, suitable for large graph matching. Extensive experiments were
conducted to demonstrate the strength of our algorithm over previous
state-of-the-art algorithms. Future work includes theoretical analysis, extension and
applications of our algorithm.

\section*{Acknowledgements}

\bibliographystyle{splncs}
\bibliography{egbib}

\section*{Appendix}
\subsection*{FastGA}
In general, GA also has time complexity $O(n^4)$ per iteration.
However, when the compatibility matrix $C=A\otimes A'$ (which was mentioned in the original paper of GA~\cite{GA}) is used, GA
has only time complexity $O(n^3)$ per iteration. To see this,
consider the most time-consuming step of GA, which has time
complexity $O(n^4)$
\begin{equation}\label{GAO4}
\forall \  i, \forall \ a, \quad X_{ia}^{(t+1)}=\exp(\beta
Q_{ia}^{(t)}), \quad
Q_{ia}^{(t)}=\sum_{b=1}^{n}\sum_{j=1}^{n}X_{bj}^{(t)}C_{aibj},
\end{equation}
where $X$ is the soft-assignment matrix. If
$C_{aibj}=A_{ij}A'_{ab}$, then (\ref{GAO4}) can be compactly written
as
\begin{equation}\label{GAO3}
X^{(t+1)}=\exp(\beta AX^{(t)}A'),
\end{equation}
which has time complexity $O(n^3)$.

\subsection*{Problem Formulation}
By expanding the Fronbenius norm, we have
\begin{equation}
\frac{1}{2}\| A-XA'X^T \|^2_F + \lambda\| B - XB' \|^2_F
\end{equation}
\begin{equation}
=\frac{1}{2}tr(AA^T) +\frac{1}{2}tr(X^TXA'X^TXA'^{T}) \label{trace1} - tr(X^{T}AXA')
\end{equation}
\begin{equation}
+\lambda tr(BB^T) + \lambda tr(X^TXB'B'^T)  - 2\lambda tr(X^TBB') \label{trace2},
\end{equation}
due to the invariance of the matrix trace under cyclic permutation. Since $X$ is a partial permutation matrix and $n \geq n'$, $X^TX = I$. And the first two terms of (\ref{trace1}) and (\ref{trace2}) are constants. Therefore, the minimization problem
\begin{equation}
\min_X \frac{1}{2}\| A- XA'X^T \|^2_F + \lambda\| B - XB' \|^2_F,
\end{equation}
\begin{equation}
s.t. \quad X \mathbf{1}\leq\mathbf{1}, X^T\mathbf{1}=\mathbf{1}, X
\in \{0,1\}^{n\times n'},
\end{equation}
is equivalent to
\begin{equation}\
\max_{X} \  \frac{1}{2}tr(X^{T}AXA') + \lambda tr(X^TK),
\end{equation}
\begin{equation}
s.t. \quad X\mathbf{1}\leq\mathbf{1}, X^T\mathbf{1}=\mathbf{1}, X
\in \{0,1\}^{n\times n'}.
\end{equation}

\subsection*{Successive Projection}
The derivation of (\ref{P1})(\ref{P2}) is due to \cite{DSN}, with
slight modification in here. We first consider $P_1$. The Lagrangian
of (\ref{PP1}) is
\begin{align}
L(D,u_1,u_2)&=tr(D^TD-2X^TD)  \\
&- u_1^T(D\textbf{1}-\textbf{1}) -
u_2^T(D^T\textbf{1}-\textbf{1}).
\end{align}
Let $u_1=u_2=u$ and set the derivative with respect to $D$ to zero,
we have
\begin{equation}
D=X+u\textbf{1}^T+\textbf{1}u^T.
\end{equation}
Multiplies by \textbf{1} on both sides:
$u=(nI+\textbf{1}\textbf{1}^T)^{-1}(I-X)\textbf{1}$. Combining the
fact that
$(nI+\textbf{1}\textbf{1}^T)^{-1}=(1/n)(I-(1/2n)\textbf{1}\textbf{1}^T)$,
we obtain (\ref{P1}). The derivation of (\ref{P2}) is
straightforward and omitted.

\begin{IEEEbiography}{Yao Lu}
\end{IEEEbiography}

\begin{IEEEbiography}{Kaizhu Huang}
\end{IEEEbiography}


\begin{IEEEbiography}{Cheng-Lin Liu}
\end{IEEEbiography}

\end{document}